\pdfoutput=1
\documentclass[journal]{IEEEtran}

\usepackage[T1]{fontenc}
\usepackage[utf8]{inputenc}
\usepackage{amsmath,amssymb,amsthm}
\usepackage{microtype}
\usepackage{cite}
\usepackage{url}

\usepackage{booktabs}
\usepackage{multirow}
\usepackage{tabularx}
\usepackage{colortbl}
\usepackage{xcolor}

\newlength{\tabnat}
\newcommand{\fullwidthtab}[2]{%
  \setlength{\tabcolsep}{0pt}\sbox0{#2}\setlength{\tabnat}{\wd0}%
  \setlength{\tabcolsep}{\dimexpr(\columnwidth-\tabnat)/\numexpr 2*#1\relax\relax}%
  \typeout{FULLWIDTHTAB natural=\the\tabnat\space sep=\the\tabcolsep}%
  #2}

\definecolor{dkgreen}{HTML}{1B6B2A}
\definecolor{dkred}{HTML}{A11212}
\definecolor{grayblue}{rgb}{0.13,0.22,0.42}
\newcommand{\cgood}[1]{\textcolor{dkgreen}{$\mathbf{#1}$}}
\newcommand{\cbad}[1]{\textcolor{dkred}{$#1$}}

\definecolor{dkamber}{HTML}{9A6B00}
\definecolor{robotblue}{HTML}{2B6E8C}   %
\definecolor{robotred}{HTML}{B4502C}
\newcommand{\cmid}[1]{\textcolor{dkamber}{$#1$}}
\newcommand{\cpart}[1]{\textcolor{dkamber}{\scriptsize(#1)}}

\usepackage{graphicx}
\usepackage{algorithm}
\usepackage[indLines=true]{algpseudocodex}
\definecolor{kwfor}{HTML}{1F77B4}
\definecolor{kwif}{HTML}{E8710A}
\definecolor{kwrep}{HTML}{8E44AD}
\definecolor{kwret}{HTML}{D62728}
\definecolor{kwend}{HTML}{7F7F7F}
\algrenewcommand\algorithmicforall{{\color{kwfor}\textbf{for all}}}
\algrenewcommand\algorithmicfor{{\color{kwfor}\textbf{for}}}
\algrenewcommand\algorithmicwhile{{\color{kwrep}\textbf{while}}}
\algrenewcommand\algorithmicif{{\color{kwif}\textbf{if}}}
\algrenewcommand\algorithmicthen{}
\algrenewcommand\algorithmicelse{{\color{kwif}\textbf{else}}}
\algrenewcommand\algorithmicend{{\color{kwend}\textbf{end}}}
\algrenewcommand\algorithmicrepeat{{\color{kwrep}\textbf{repeat}}}
\algrenewcommand\algorithmicuntil{{\color{kwrep}\textbf{until}}}
\algrenewcommand\algorithmicreturn{{\color{kwret}\textbf{return}}}
\makeatletter
\patchcmd{\algpx@drawIndentLine}
  {\draw[algpxIndentLine]}
  {\draw[algpxIndentLine,algpxIndentLine#1/.try]}
  {\typeout{ALGPX-PATCH-OK}}{\typeout{ALGPX-PATCH-FAILED}}
\makeatother
\tikzset{
  algpxIndentLine0/.style={draw=kwfor},
  algpxIndentLine1/.style={draw=kwrep},
  algpxIndentLine2/.style={draw=kwif},
}
\usepackage{placeins}

\usepackage[capitalize,nameinlink]{cleveref}

\theoremstyle{definition}
\newtheorem{definition}{Definition}

\newtheorem{remark}{Remark}
\newtheorem{example}{Example}
\theoremstyle{plain}
\newtheorem{lemma}{Lemma}
\newtheorem{theorem}{Theorem}
\newtheorem{proposition}{Proposition}

\newcommand{\X}{\mathcal{X}}
\newcommand{\A}{\mathcal{A}}
\newcommand{\R}{\mathbb{R}}
\newcommand{\Z}{\mathbb{Z}}
\newcommand{\tra}{\pi}
\newcommand{\cl}{\mathcal{T}}

\newcommand{\res}{\mathcal{R}}
\newcommand{\preds}{\mathcal{G}}

\newcommand{\occ}{\operatorname{occ}}
\newcommand{\cells}{\mathcal{K}}
\newcommand{\roll}{\operatorname{roll}}

\newcommand{\jet}{\operatorname{jet}}
\newenvironment{tracelist}{\begin{list}{}{%
  \setlength{\leftmargin}{0pt}%
  \setlength{\itemindent}{0pt}%
  \setlength{\labelwidth}{0pt}%
  \setlength{\labelsep}{0.4em}%
  \setlength{\topsep}{2pt}%
  \setlength{\itemsep}{2pt}%
  \setlength{\parsep}{0pt}%
  }}%
  {\end{list}}

\title{Dispersive Tree Agreements for Multi-Robot Coordination}

\author{Shashank A. Deshpande, Blake T. Edwards, and Jonathan P. How%
\thanks{The authors are with the Department of Aeronautics and Astronautics,
Massachusetts Institute of Technology, Cambridge, MA 02139, USA
(e-mail: croshank@mit.edu; blakete@mit.edu; jhow@mit.edu).}}

\begin{document}
\maketitle

\begin{abstract}
Multi-robot coordination of nonlinear systems is naturally
formulated as joint-space optimal control subject to joint constraints. Existing methods that approach joint
optimality remain largely centralized and offline, while decentralized, real-time
counterparts target safe navigation without pursuing the joint
optimum. This paper develops a decentralized, real-time planner
that approaches the joint optimum as an anytime process. We extend
a recent single-robot framework that searches dispersive trees of
forward-propagated commands to
coordinate a multi-robot team near-optimally. We first show that the inter-robot coupling can be effectively
replaced by some local constraints, under which locally grown
dispersive trees contain a near-optimal agreement. It is this
structure that the robots pursue by a negotiation driven by decentralized column
generation. In the quiescent limit, this negotiation is shown to settle on a
near-optimal agreement. Our real-time protocol, WarpWeaver, runs the
negotiation asynchronously over a receding horizon and closes on quiesced
agreements even when deadlines cut it short. In
simulation, WarpWeaver reduces total travel time by up to $22\%$
against centralized offline planners on the Dynobench suite and by up to $16\%$ against a
decentralized real-time baseline, and outperforms a task-specific planner on networked navigation. Hardware
experiments on six robots planning onboard demonstrate real-world
implementability.
\end{abstract}


\section{Introduction}
\label{sec:intro}

This paper addresses the coordination of a team of robots with
differentially flat nonlinear dynamics, formulated as optimal control in the joint
space of the team subject to constraints over groups of robots.
We seek control inputs
whose cost lies near the joint optimum, via
distributed computation across the robots in real time, to
maintain robustness to changes in the individual objectives of
the robots as well as in the dynamics of the environment. Existing approaches to this problem
divide broadly into centralized, offline planners with
guarantees of asymptotic optimality, and real-time decentralized
planners that work with simplified dynamics, without pursuing
the joint optimum.

A range of centralized planners derive from multi-agent path
finding \cite{sharon2015cbs} and handle pairwise conflicts in a
high-level search on top of low-level single-robot planners
running model predictive control \cite{tajbakhsh2024mpccbs},
sampling-based planning \cite{kottinger2022kcbs}, or search over
graphs of discontinuity-bounded motion primitives
\cite{moldagalieva2024dbcbs,moldagalieva2025dbecbs,moldagalieva2025dblacam}.
Sampling-based planners search the composite configuration
space of the team \cite{cap2013marrtstar,shome2020drrtstar}. These planners provide
asymptotic guarantees of completeness or optimality, and focus
on motion planning with pairwise separation.

Decentralized planning, on the other hand, comes with real-time
advantages such as robustness to changes in the individual
objectives or in the environment. Decentralized trajectory
optimization involves
high-frequency replanning of safe splines around the incumbents
communicated by peer robots, while often simplifying the
robots' dynamics \cite{zhou2021egoswarm,tordesillas2022mader,kondo2023rmader,
luis2019dmpc,luis2020online,park2022lsc}. Consensus-based
distributed optimization instead lets the robots negotiate, with
gradient-based local optimization driven by penalties on
mutual disagreement \cite{ferranti2022distributed,saravanos2023distributed}.
Coordination across these methods, however, often settles in
local optima,
without pursuing the joint optimum.

\smallskip
In summary, real-time multi-robot coordination deals with multiple
sources of complexity. First, planning and control at the level
of a single robot must respect the nonlinear dynamics of its
control system. In addition, these
control trajectories must coordinate to satisfy various
constraints over groups of robots while evading the many local
minima of the joint problem.
Finally, the algorithm that addresses these challenges must
decentralize across the robots under limited communication and
maintain robustness to communication delays within the team.
While existing approaches handle this complexity in part, this
paper develops a decentralized framework that tackles each of its
sources in unison.

\begin{figure}[t]
\centering
\includegraphics[width=\columnwidth]{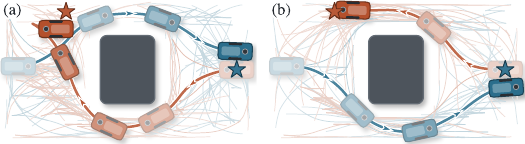}
\caption{(a) Iterated Best Response,
$F = \textcolor{robotblue}{7.5} \;{+}\; \textcolor{robotred}{10.0} = 17.5$\,s,
vs.\ (b) Dispersive Tree Agreement,
$F = \textcolor{robotblue}{8.0} \;{+}\; \textcolor{robotred}{5.5} = 13.5$\,s,
where the total travel time $F$ sums the arrival times of the blue
and red robots.}
\label{fig:teaser}
\end{figure}

\smallskip
We employ dispersive forward tree search \cite{deshpande2026DFT},
which finds near-optimal control inputs for differentially flat
systems, for single-robot planning and control. To solve coordination, we first show that the coupling among the
robots can be absorbed into suitable constraints on each robot
alone, so that tractably sized dispersive trees grown locally by the
robots jointly contain a near-optimal agreement. It is this structure that the
remainder of our development pursues. \Cref{fig:teaser} contrasts such an agreement with
iterated best response~\cite{zhou2021egoswarm,tordesillas2022mader,kondo2023rmader},
which can settle on a sub-optimal equilibrium. We formulate the selection of this
joint agreement as an integer linear program and approach its
solution by column generation
\cite{dantzig1960decomposition,lubbecke2005column}, in which the
dual of its linear relaxation sets rents on the shared space.
Under these rents, the search for a new column becomes a
single-robot problem that each robot solves near-optimally with
its own dispersive forward tree.
Notably, this column generation procedure is effectively
decentralized, to distribute the computational load of the
process across the robots while limiting communication within
the team to a manageable amount, as the robots exchange only the command
sequences they generate and their endorsements, never their trees. We show that the resulting process implements an
anytime negotiation whose agreement is certified near-optimal
once it reaches quiescence, up to the resolution of the search and a gap that the agents compute online. This development culminates in
WarpWeaver, the real-time multi-robot planner that we propose,
which runs the negotiation asynchronously across the team over a
receding horizon and remains robust to bounded communication
delays. Our experiments thereby demonstrate that this negotiation,
driven by decentralized column generation, is a pattern well
suited to real-time multi-robot coordination, as the
agreements that it commits within real-time deadlines come close in
cost to those that it reaches when given time to quiesce.

\smallskip
In summary, this paper makes the following contributions.

\noindent $\bullet$~\textbf{Joint optimal control as distributed
dispersive forward search.} We cast a broad class of multi-robot
coordination problems as joint-space optimal control of
differentially flat nonlinear systems, and prove that suitably constrained, locally grown
dispersive forward trees contain a near-optimal agreement
(\cref{prop:jcs}).

\noindent $\bullet$~\textbf{Decentralized column generation.} We
pursue this structure by decentralized column generation, in which rents on the shared space guide the
tree search of each robot, and prove that the negotiated agreement
is certified near-optimal at quiescence, relative to the joint
solutions robust at the search resolution
(\cref{prop:priced,thm:certificate}).

\noindent $\bullet$~\textbf{Real-time protocol.} We develop
WarpWeaver, an asynchronous protocol that runs decentralized
column generation over a receding horizon onboard each robot, to
coordinate the team in real time under bounded communication
delays, whose agreements within real-time deadlines come
close in cost to quiesced ones.

\noindent $\bullet$~\textbf{Evaluation and hardware
experiments.} We benchmark WarpWeaver in simulation against
centralized offline, decentralized real-time, and task-specific
planners. It reduces total travel time by up to $22\%$ against
centralized offline planners on Dynobench and by up to $16\%$ against
a decentralized real-time baseline. Further, on a graph connectivity
task, it nearly doubles the success rate of a task-specific planner. We further deploy it on a team of six robots, each planning
onboard in real time.

\smallskip
The remainder of this paper is organized as follows.
\Cref{sec:relwork} reviews related work. \Cref{sec:problem}
formulates the joint optimal control problem and its flat
reformulation, and \cref{sec:trees} reviews dispersive forward
tree search. \Cref{sec:agreement} develops the decentralized
column generation scheme, its near-optimality certificate, and the
WarpWeaver protocol. \Cref{sec:sim} presents the simulation
experiments and benchmarks, and \cref{sec:hardware} the real-robot
experiments. We conclude in \cref{sec:conclusion}. The proofs are
deferred to \cref{app:proofs}.

\section{Related Work}
\label{sec:relwork}

In this section, we present a brief review of existing literature
on multi-robot planning and control, along with related techniques
relevant to our work.

\subsection{Sampling and Search for Multi-Robot Motion Planning}
\label{sec:relwork-search}

Sampling-based planning extends to multi-robot motion planning in
the composite configuration space of the
team~\cite{cap2013marrtstar,shome2020drrtstar}. For geometric paths,
search over the tensor product of individual roadmaps has been
shown to be asymptotically optimal~\cite{shome2020drrtstar}, and near-optimal at
a finite sample size~\cite{dayan2023nearoptimal}. While it is
unclear how this decomposition carries over to robots with
continuous dynamics, kinodynamic tree search in the composite state
space has been shown to be probabilistically
complete~\cite{le2018cooperative}.
Decentralized implementations of composite search have also been
demonstrated, with robots replanning in turn by merit-based token
passing~\cite{desaraju2012dmarrt} or exchanging probability
distributions over their plans~\cite{best2019decmcts}.

Kinodynamic extensions of multi-agent path finding on
graphs~\cite{sharon2015cbs} retain its two-level structure, a
high-level search that coordinates the team on top of low-level
local planners. The low-level graph search is then replaced by a range of methods spanning sampling-based tree
search~\cite{kottinger2022kcbs}, model predictive
control~\cite{tajbakhsh2024mpccbs}, and search over graphs of motion
primitives joined up to a bounded
discontinuity~\cite{moldagalieva2024dbcbs,moldagalieva2025dbecbs,moldagalieva2025dblacam}.
These implementations are predominantly centralized and offline,
and guarantee solution properties ranging from probabilistic or
resolution completeness to asymptotic optimality.

\subsection{Trajectory Optimization for Multi-Robot Motion Planning}
\label{sec:relwork-opt}

Centralized trajectory optimization coordinates the team within a
single program~\cite{mellinger2012miqp,augugliaro2012scp}, whose size grows with the number of robots.
Scalable variants involve a discrete coordination layer, such as
priority-based search over mixed-integer programs for tracked
reference paths~\cite{chen2021s2m2} and multi-agent path finding
over B\'ezier curves~\cite{yan2024bezier}.
Decentralized trajectory optimization segregates the program across
the robots, each of which replans a trajectory around the incumbent
commitments of its peers. These commitments enter as collision
penalties in gradient-based optimization~\cite{zhou2021egoswarm},
or as separating constraints robust to imperfect
communication~\cite{tordesillas2022mader,kondo2023rmader}.
Distributed model predictive control~\cite{luis2019dmpc,luis2020online}
follows a similar pattern over short horizons, with linear safe
corridors~\cite{park2022lsc} preserving feasibility. These planners run in real time at scale, typically on
simplified models of the dynamics.

Instead of iterative replanning around incumbent commitments,
consensus-based distributed optimization lets the robots
negotiate. Each robot keeps local copies of the shared variables
and optimizes them under its nonlinear dynamics.
Penalties on mutual disagreement, updated by the alternating
direction method of multipliers, drive the team toward agreement
on these variables~\cite{ferranti2022distributed,saravanos2023distributed}.

\subsection{Coordination Beyond Inter-Robot Separation}
\label{sec:relwork-tasks}

Several tasks of utility constrain groups of robots beyond pairwise
separation, such as maintaining a connected communication
network~\cite{zavlanos2011graph} or keeping robots tethered to
each other. Existing work typically addresses these
tasks by deploying dedicated mechanisms. For instance, connectivity
maintenance is commonly enforced
by potential functions that keep the algebraic connectivity of the
communication graph positive~\cite{zavlanos2011graph}, a design
that Realm~\cite{bai2025realm,bai2026polygonal} extends to
line-of-sight constraints derived in real time from onboard point
clouds. Our method is able to address
coordination beyond collision avoidance by specification of hard
constraints on the joint state, as we demonstrate by
implementing connectivity maintenance in simulation and tethering
in hardware experiments.

\subsection{Column Generation and Decentralized Pricing}
\label{sec:relwork-cg}

Column generation is a canonical technique for solving large linear
programs by generating their variables on
demand~\cite{dantzig1960decomposition,lubbecke2005column}. Dual
prices from a restricted master program guide the search for new
columns, with a Lagrangian bound on the gap~\cite{farley1990note}
and branching for integrality~\cite{barnhart1998price}.
This technique has been applied to
multi-agent path finding on graphs by
branch-and-cut-and-price~\cite{lam2022bcp}, and to multi-robot
routing over space-time positions with pricing by
resource-constrained shortest paths~\cite{haghani2020integer}.
Column generation has been shown to distribute effectively across
robots, with local pricing
and an exchange of columns among neighbors, for task
assignment~\cite{testa2022generalized}.

We formulate joint optimal control as a search over the dispersive
forward trees of the robots~\cite{deshpande2026DFT} and distribute
this search across the team via column generation, running
asynchronously in real time with a certificate of near-optimality
at quiescence, relative to the joint solutions that are robust at the search resolution.

\section{Problem Formulation}
\label{sec:problem}

Consider a team of $N$ agents, agent $i$ governed by the
nonlinear control system
\begin{equation}
\label{eq:sys}
  \dot{x}^{i}(t) = f^{i}\bigl(x^{i}(t), u^{i}(t)\bigr),
\end{equation}
with state $x^{i} \in \X_{i} \subset \R^{n_{i}}$ and control
input $u^{i} \in U_{i} \subset \R^{m_{i}}$. We address the
coordination of the team over a horizon $T$, formalized as
the joint optimal control problem (JOCP)
\begin{equation}
\label{eq:jocp}
\begin{aligned}
  \min_{u^{1:N}(\cdot)} \quad
  & \Phi\bigl( x^{1}(\cdot),\, \dots,\, x^{N}(\cdot) \bigr)
    \\[0.2em]
  \text{s.t.} \quad
  & \dot{x}^{i}(t) = f^{i}\bigl(x^{i}(t), u^{i}(t)\bigr),
    \qquad t \in [0, T], \\[0.2em]
  & g\bigl( x^{i}(\cdot) \bigr)
    \ge 0 \qquad \forall\, g \in \preds_{i}, \\[0.2em]
  & g\bigl( ( x^{j}(\cdot) )_{j \in e} \bigr)
    \ge 0 \qquad \forall\, g \in \preds_{e},
    \ \forall\, e, \\[0.2em]
  & x^{i}(0) = x^{i}_{\mathrm{init}},
    \qquad i = 1, \dots, N ,
\end{aligned}
\end{equation}
in which $\Phi$ is the joint cost, and the trajectories
are constrained by \emph{individual} predicates $\preds_{i}$
and \emph{joint} predicates $\preds_{e}$ over groups
$e \subset \{ 1, \dots, N \}$ of agents. Without loss of
generality, let $T \in \mathbb{N}$, dividing the horizon
into the unit time windows $[k, k+1]$, $0 \le k < T$, and
let $d( \cdot, \cdot )$ denote the supremum metric on
trajectories over a window. We make the following
assumptions on the structure of the coordination problem.

\smallskip
\noindent (i) \textbf{Cost function:} The joint cost
is a monotone, convex, piecewise-linear aggregate of private
agent-wise costs $J_{i} : C([0, T]; \X_{i}) \to \R$, i.e.,
for $\alpha_{\ell i} \ge 0$,
\begin{equation}
\label{eq:aggregator}
  \Phi\bigl( x^{1}(\cdot), \dots, x^{N}(\cdot) \bigr)
  = \max_{1 \le \ell \le L}\
  \textstyle\sum_{i} \alpha_{\ell i}\, J_{i}\bigl( x^{i}(\cdot)
  \bigr) .
\end{equation}
Each agent-wise cost is in turn additive over the windows,
with $L_{\varrho}$-Lipschitz stage costs
$\varrho^{i}_{k} : C([k, k+1]; \X_{i}) \to \R$,
\begin{equation}
\label{eq:winagg}
\begin{gathered}
  J_{i}\bigl( x^{i}(\cdot) \bigr)
  = \sum_{k=0}^{T-1}
  \varrho^{i}_{k}\bigl( x^{i}|_{[k,\, k+1]} \bigr) , \\
  \bigl| \varrho^{i}_{k}( w ) - \varrho^{i}_{k}( w' ) \bigr|
  \;\le\; L_{\varrho}\, d( w, w' ) ,
\end{gathered}
\end{equation}
for all window trajectories $w, w'$ of agent $i$. The
Lipschitz constant of the cost $J_{i}$ over the horizon is
then $L_{J} := T\, L_{\varrho}$.

\smallskip
\noindent (ii) \textbf{Individual predicates:} The
predicates $g \in \preds_{i}$ decompose windowwise into
Lipschitz predicates $g_{k} : C([k, k+1]; \X_{i}) \to \R$
on the window trajectories, of which they are the minimum,
\begin{equation}
\label{eq:winsep}
\begin{gathered}
  g\bigl( x^{i}(\cdot) \bigr)
  = \min_{0 \le k < T}
  g_{k}\bigl( x^{i}|_{[k,\, k+1]} \bigr) , \\
  \bigl| g_{k}( w ) - g_{k}( w' ) \bigr|
  \le L_{\preds}\, d( w, w' ) .
\end{gathered}
\end{equation}

\smallskip
\noindent (iii) \textbf{Joint predicates:} Each
$g \in \preds_{e}$ is the running minimum over time of a
function of the states of the group,
\begin{equation}
\label{eq:instant}
  g\bigl( ( x^{j}(\cdot) )_{j \in e} \bigr)
  = \min_{t \in [0, T]}
  h_{t}\bigl( ( x^{j}(t) )_{j \in e} \bigr) ,
\end{equation}
in which $h_{t} : \prod_{j \in e} \X_{j} \to \R$ is
$L_{\preds}$-Lipschitz, uniformly in $t$, with respect to
the maximum of the members' distances.

\smallskip
While the supposed structure of the JOCP is rather abstract,
it allows our treatment to admit coordination problems of
generic interest, as the examples of the following section
illustrate.

\subsection{Examples}
\label{sec:examples}

\begin{example}[Multi-robot kinodynamic motion planning]
\label{ex:kino}
Given, for each agent $i$, a time-varying free space
$\X^{i}_{\mathrm{free}}(t) \subseteq \X_{i}$, a goal region
$\X^{i}_{\mathrm{goal}} \subseteq \X_{i}$, and an initial
state $x^{i}_{\mathrm{init}}$, the multi-robot kinodynamic
motion planning problem over the horizon $T$ is
\begin{equation}
\label{eq:kino}
\begin{aligned}
  \min_{u^{1:N}(\cdot)} \quad
  & \textstyle\sum_{i}
    \operatorname{dist}\bigl( x^{i}(T),\,
    \X^{i}_{\mathrm{goal}} \bigr)
    \\[0.2em]
  \text{s.t.} \quad
  & \dot{x}^{i}(t) = f^{i}\bigl(x^{i}(t), u^{i}(t)\bigr),
    \qquad t \in [0, T], \\[0.2em]
  & x^{i}(t) \in \X^{i}_{\mathrm{free}}(t),
    \qquad t \in [0, T], \\[0.2em]
  & d_{\mathrm{sep}}\bigl( x^{i}(t),\, x^{j}(t) \bigr) \ge 0,
    \quad t \in [0, T], \ \forall\, j \ne i, \\[0.2em]
  & \bigl| v\bigl( x^{i}(T) \bigr) \bigr|
    \le \varepsilon_{v},
    \qquad x^{i}(0) = x^{i}_{\mathrm{init}},
\end{aligned}
\end{equation}
where $\operatorname{dist}( x, Y )$ denotes the distance of
the state $x$ from a set $Y$, $v( x )$ denotes the velocity
components of the state, so that the terminal constraint
requires each agent to end the horizon at approximate rest,
$\pi( x )$ denotes the position components of the state,
and $d_{\mathrm{sep}}( x^{i}, x^{j} ) = | \pi( x^{i} ) -
\pi( x^{j} ) | - R_{i} - R_{j}$ separates the bodies of
radii $R_{i}$, $R_{j}$ of the agents.
This problem is an instance of \eqref{eq:jocp} with the sum
of costs as the aggregator, the terminal costs
\[
  J_{i}\bigl( x^{i}(\cdot) \bigr)
  = \operatorname{dist}\bigl( x^{i}(T),\,
  \X^{i}_{\mathrm{goal}} \bigr) ,
\]
the individual predicates
$\preds_{i} = \{ g^{i}_{\mathrm{free}},\,
g^{i}_{\mathrm{rest}} \}$ encoding the free space and
terminal rest,
\begin{gather*}
  g^{i}_{\mathrm{free}}\bigl( x^{i}(\cdot) \bigr)
  = \min_{t \in [0, T]}
  \operatorname{sd}\bigl( x^{i}(t),\,
  \X_{i} \setminus \X^{i}_{\mathrm{free}}(t) \bigr) , \\
  g^{i}_{\mathrm{rest}}\bigl( x^{i}(\cdot) \bigr)
  = \varepsilon_{v}
  - \bigl| v\bigl( x^{i}(T) \bigr) \bigr| ,
\end{gather*}
with $\operatorname{sd}( x, Y )$ the signed distance of $x$
from the set $Y$, and, for each pair $e = \{ i, j \}$, the
separation predicate
$\preds_{e} = \{ g^{ij}_{\mathrm{sep}} \}$ of the form
\eqref{eq:instant} with
$h( x^{i}, x^{j} ) = | \pi( x^{i} ) - \pi( x^{j} ) | - R_{i} - R_{j}$.
The costs and the predicates satisfy (i) to (iii) with
$L_{\varrho} = 1$ and $L_{\preds} = 2$.
\end{example}

\begin{example}[Connectivity maintenance]
\label{ex:relay}
Let a base station be located at $y_{0} \in \R^{3}$, and let
two nodes communicate when they are within a range $R_{c}$
of each other and their connecting segment clears the
obstacle set $O \subset \R^{3}$ by a distance $\delta_{c}$.
The team is to reach its goal regions while the
communication graph on the team and the base stays
connected, with no edge prescribed,
\begin{equation}
\label{eq:relay}
\begin{aligned}
  \min_{u^{1:N}(\cdot)} \quad
  & \textstyle\sum_{i}
    \operatorname{dist}\bigl( x^{i}(T),\,
    \X^{i}_{\mathrm{goal}} \bigr)
    \\[0.2em]
  \text{s.t.} \quad
  & \dot{x}^{i}(t) = f^{i}\bigl(x^{i}(t), u^{i}(t)\bigr),
    \qquad t \in [0, T], \\[0.2em]
  & x^{i}(t) \in \X^{i}_{\mathrm{free}}(t),
    \qquad t \in [0, T], \\[0.2em]
  & d_{\mathrm{sep}}\bigl( x^{i}(t),\, x^{j}(t) \bigr) \ge 0,
    \quad t \in [0, T], \ \forall\, j \ne i, \\[0.2em]
  & G\bigl( x^{1}(t), \dots, x^{N}(t) \bigr)
    \ \text{connected},
    \quad t \in [0, T], \\[0.2em]
  & \bigl| v\bigl( x^{i}(T) \bigr) \bigr|
    \le \varepsilon_{v},
    \qquad x^{i}(0) = x^{i}_{\mathrm{init}},
\end{aligned}
\end{equation}
where $G$ is the graph on the nodes $0, 1, \dots, N$, with
$\pi( x^{0} ) \equiv y_{0}$ for the base, that joins $i$ and
$j$ when their link margin
\begin{gather*}
  \ell( x^{i}, x^{j} )
  = \min\bigl\{\, R_{c} - | \pi( x^{i} ) - \pi( x^{j} ) |, \\
  \qquad\qquad
  \operatorname{dist}\bigl( [ \pi( x^{i} ), \pi( x^{j} ) ],\,
  O \bigr) - \delta_{c} \,\bigr\}
\end{gather*}
is nonnegative. This problem is an instance of
\eqref{eq:jocp} with the costs and the individual and
pairwise predicates of \cref{ex:kino}, and one connectivity
predicate on the group of all agents, of the form
\eqref{eq:instant} with
\[
  h( x^{1}, \dots, x^{N} )
  = \max_{S \in \mathcal{S}}\ \min_{\{ i, j \} \in S}
  \ell( x^{i}, x^{j} ) ,
\]
where $\mathcal{S}$ is the set of spanning trees of the
complete graph on the nodes $0, 1, \dots, N$. The margin
$h$ is nonnegative exactly when some spanning tree has all
of its links alive, that is, when $G$ is connected, and is
attained by a maximum spanning tree under the weights
$\ell$. The link margin is Lipschitz since the distance of
the segment from the obstacles is a minimum over convex
combinations of its endpoints, and $h$, a maximum of minima
of link margins, inherits its constant, $L_{\preds} = 2$.
The predicate couples the whole team: no single link is
required, and a link may break whenever the others keep the
graph connected. Fixing the tree $S$ recovers the
maintenance of a prescribed graph, the pairwise case, and
the path from the base through agents $1, \dots, N$ is a
relay chain.
\end{example}

\subsection{Flat Reformulation}
\label{sec:flatform}

We restrict the dynamics $f^{i}$ of each agent to the
sufficiently regular differentially flat structure of
\cite{deshpande2026DFT}. In particular, we make the following
assumptions about the agents' dynamics.

\smallskip
\noindent (a) \textbf{Regularity:} The dynamics $f^{i}$ is
regular enough that, starting from any state $x \in \X_{i}$,
each admissible control $u^{i} : [0, T] \to U_{i}$ generates
a unique trajectory of \eqref{eq:sys}.

\smallskip
\noindent (b) \textbf{Differential flatness with indices
$R^{i}$:} The system \eqref{eq:sys} is differentially flat
in the sense of \cite{deshpande2026DFT}, with a
state-determined flat output
$z^{i} = h^{i}( x^{i} ) \in \R^{d_{i}}$ and indices
$R^{i} = ( r^{i}_{1}, \dots, r^{i}_{d_{i}} )$,
$| R^{i} | := \sum_{l} r^{i}_{l}$, on a fixed open domain
of regular points, which excludes the singular locus of the
flatness maps~\cite{deshpande2026DFT}. Denote
\begin{equation}
\label{eq:jet}
  \jet z^{i}
  := \bigl( z^{i}_{l},\, \dot{z}^{i}_{l},\, \dots,\,
  z^{i\, ( r^{i}_{l} - 1 )}_{l} \bigr)_{l=1}^{d_{i}} ,
  \quad
  z^{i\, (R^{i})}
  := \bigl( z^{i\, ( r^{i}_{l} )}_{l} \bigr)_{l=1}^{d_{i}} .
\end{equation}
The flatness maps accordingly take the continuously
differentiable form
\begin{equation}
\label{eq:flatmaps}
  x^{i}(t) = \Psi^{i}_{x}\bigl( \jet z^{i}(t) \bigr) ,
  \qquad
  u^{i}(t) = \Psi^{i}_{u}\bigl( \jet z^{i}(t),\,
  z^{i\, (R^{i})}(t) \bigr)
\end{equation}
along every trajectory. Conversely, let $z^{i}(\cdot)$ be a
curve with components $z^{i}_{l}$ in the Sobolev space
$W^{r^{i}_{l}, \infty}( [0, T]; \R )$ of scalar functions
with essentially bounded derivatives up to order
$r^{i}_{l}$, and suppose that
$\Psi^{i}_{u}( \jet z^{i}(t),\, z^{i\, (R^{i})}(t) ) \in U_{i}$
for almost every $t \in [0, T]$. Then
$z^{i} = h^{i}( x^{i}(\cdot) )$ for a trajectory
$x^{i}(\cdot)$ of \eqref{eq:sys}, realized by the control
$u^{i}(t) = \Psi^{i}_{u}( \jet z^{i}(t),\, z^{i\, (R^{i})}(t) )$
from the state $x^{i}(0) = \Psi^{i}_{x}( \jet z^{i}(0) )$.

\smallskip
\noindent (c) \textbf{Compact input set:} The input set
$U_{i} \subset \R^{m_{i}}$ is compact.

\smallskip
\noindent (d) \textbf{Jet observability:} There is a map
$\Xi^{i}$ such that, along every trajectory of
\eqref{eq:sys}, $\jet z^{i}(t) = \Xi^{i}( x^{i}(t) )$. An
extension of the state renders the jet observable whenever
required \cite{deshpande2026DFT}.

\smallskip
By flatness (b) and jet observability (d), the state along a
trajectory is the function
$x^{i}(t) = \Psi^{i}_{x}( \jet z^{i}(t) )$ of the output jet,
so the cost and the predicates of \eqref{eq:jocp}
induce functionals of the output jet curves, which we
continue to denote by $\Phi$, $J_{i}$, and $g$. Moreover, by the flatness
converse in (b), the trajectories of admissible controls from
$x^{i}_{\mathrm{init}}$ are exactly the state realizations
of the output curves $z^{i}(\cdot)$,
$z^{i}_{l} \in W^{r^{i}_{l}, \infty}( [0, T]; \R )$, whose
induced control lies in $U_{i}$ almost everywhere and that
start at the initial jet. The problem \eqref{eq:jocp} is
therefore equivalent to the output-space problem
\begin{equation}
\label{eq:jocp-flat}
\begin{aligned}
  \min_{z^{1:N}(\cdot)} \
  & \Phi\bigl( \jet z^{1}(\cdot),\, \dots,\,
    \jet z^{N}(\cdot) \bigr)
    \\[0.2em]
  \text{s.t.} \
  & \Psi^{i}_{u}\bigl( \jet z^{i}(t),\,
    z^{i\, (R^{i})}(t) \bigr) \in U_{i}
    \ \ \text{a.e.\ on } [0, T], \\[0.2em]
  & g\bigl( \jet z^{i}(\cdot) \bigr) \ge 0
    \qquad \forall\, g \in \preds_{i} , \\[0.2em]
  & g\bigl( ( \jet z^{j}(\cdot) )_{j \in e} \bigr) \ge 0
    \qquad \forall\, g \in \preds_{e} ,
    \ \forall\, e , \\[0.2em]
  & \jet z^{i}(0) = \Xi^{i}\bigl( x^{i}_{\mathrm{init}}
    \bigr),
    \qquad i = 1, \dots, N .
\end{aligned}
\end{equation}
Note that the structural assumptions (i) to (iii) of
\cref{sec:problem} transfer to this output-space
reformulation. Indeed, the state maps $\Psi^{i}_{x}$ are
continuously differentiable, hence Lipschitz on compact
sets with a common constant $L_{\mathbf{x}}$, so that for
all window jet curves $\zeta, \zeta'$ of agent $i$, in the
supremum metric over the window induced by the maximum norm on the jet
space,
\begin{equation}
\label{eq:jetlip}
  d\bigl( \Psi^{i}_{x} \circ \zeta,\,
  \Psi^{i}_{x} \circ \zeta' \bigr)
  \le L_{\mathbf{x}}\, d( \zeta, \zeta' ) .
\end{equation}
The stage costs, the window predicates, and the joint
predicates induce
\begin{equation}
\label{eq:induced}
\begin{gathered}
  \tilde{\varrho}^{i}_{k}( \zeta )
  := \varrho^{i}_{k}\bigl( \Psi^{i}_{x} \circ \zeta \bigr) ,
  \qquad
  \tilde{g}_{k}( \zeta )
  := g_{k}\bigl( \Psi^{i}_{x} \circ \zeta \bigr) , \\
  \tilde{h}_{t}\bigl( ( \zeta^{j} )_{j \in e} \bigr)
  := h_{t}\bigl( ( \Psi^{j}_{x} \circ \zeta^{j} )_{j \in e}
  \bigr) ,
\end{gathered}
\end{equation}
which reproduce \eqref{eq:winagg}, \eqref{eq:winsep}, and
\eqref{eq:instant} on the jet curves. By \eqref{eq:jetlip},
\begin{equation}
\label{eq:jetconst}
\begin{gathered}
  \bigl| \tilde{\varrho}^{i}_{k}( \zeta )
  - \tilde{\varrho}^{i}_{k}( \zeta' ) \bigr|
  \le L_{\varrho} L_{\mathbf{x}}\, d( \zeta, \zeta' ) , \\
  \bigl| \tilde{g}_{k}( \zeta ) - \tilde{g}_{k}( \zeta' )
  \bigr|
  \le L_{\preds} L_{\mathbf{x}}\, d( \zeta, \zeta' ) ,
\end{gathered}
\end{equation}
and likewise for $\tilde{h}_{t}$ in each of its arguments.
We henceforth let $d$ denote this metric on jet curves over a
window, as well as the maximum-norm distance between jets, absorb the factor $L_{\mathbf{x}}$
into $L_{\varrho}$ and $L_{\preds}$, and drop the tildes.
Our development henceforth addresses \eqref{eq:jocp} through the flat reformulation \eqref{eq:jocp-flat}.

\section{Dispersive Forward Trees and Near-Optimal Control}
\label{sec:trees}

In this section, we review the dispersive forward tree
(DFT) search approach of \cite{deshpande2026DFT} for near-optimal
solutions of optimal control problems in the flat form of
\cref{sec:flatform}. The development reviewed concerns a
single agent system $\dot{x} = f( x, u )$ with state
$x \in \X \subset \R^{n}$ and control $u \in U \subset \R^{m}$,
with flat dynamics $f$ satisfying (a) to (d) of flatness
index $R = ( r_{1}, \dots, r_{d} )$ and flat output
$z = h( x ) \in \R^{d}$. Let $\tra^{x}_{u}$ denote the
trajectory generated by an admissible control $u$ from a
state $x$, and $\jet z_{\tra} := \Xi( \tra( \cdot ) )$ the
output jet curve of a trajectory $\tra$.

\begin{definition}[Smooth and eroded trajectory classes]
\label{def:classes}
\noindent (i) For $M > 0$, a trajectory is
\emph{$M$-smooth} if the derivatives of its output, up to the flat input, are
bounded by $M$,
\[
  \bigl| z^{(k)}_{l}(t) \bigr| \le M ,
  \qquad 1 \le k \le r_{l}, \quad 1 \le l \le d .
\]

\smallskip
\noindent (ii) For $\varepsilon > 0$, the
\emph{$\varepsilon$-eroded} input set is
\[
  U^{\varepsilon}
  := \bigl\{\, u \in U \;:\;
  u + [ -\varepsilon, \varepsilon ]^{m} \subseteq U \,\bigr\} ,
\]
and $\mathcal{U}^{\varepsilon}_{M}( x )$ denotes the class
of controls taking values in $U^{\varepsilon}$ a.e.\ whose
trajectory from $x$ is $M$-smooth.
\end{definition}

Let $L_{\Psi}$ be a Lipschitz constant of $\Psi_{u}$ on
the $2M$-smooth trajectory class, and fix
$\varepsilon \in ( 0, L_{\Psi} M ]$.

\begin{definition}[Dispersive command sets and trees]
\label{def:mincost}
Fix a dispersion $\delta_{\A} > 0$ and a pruning radius
$s > 0$, and let
$q_{s}( \xi ) := s \lfloor \xi / s \rfloor$ denote the
componentwise cell quantization on $\R^{|R|}$.

\smallskip
\noindent (i) A \emph{command set} $\A$ is a finite set of
commands $a : \X \to U^{[0, 1]}$, each assigning to a state
$x \in \X$ an admissible control signal
$a( x ) : [0, 1] \to U$. We say $\A$ is
\emph{$\delta_{\A}$-dispersive} if, for any state
$x \in \X$ and any
$u \in \mathcal{U}^{\varepsilon / 2}_{2M}( x )$,
\[
  \min_{a \in \A}\;
  d\bigl( \jet z_{\tra^{x}_{a(x)}},\;
  \jet z_{\tra^{x}_{u}} \bigr)
  \;\le\; \delta_{\A} .
\]

\smallskip
\noindent (ii) Given the initial state $x_{\mathrm{init}}$,
a command sequence
$\mathbf{a} = ( a_{1}, \dots, a_{k} ) \in \A^{k}$ generates
the trajectory $\tra_{\mathbf{a}}$ on $[0, k]$ following
\[
\begin{gathered}
  \tra_{\mathbf{a}}( j - 1 + t )
  = \tra^{x_{j-1}}_{a_{j}( x_{j-1} )}( t ),
  \qquad t \in [0, 1] , \\
  x_{0} := x_{\mathrm{init}} , \qquad
  x_{j} := \tra_{\mathbf{a}}( j ) , \qquad
  j = 1, \dots, k ,
\end{gathered}
\]
while
$\roll( \mathbf{a} ) := \jet z_{\tra_{\mathbf{a}}}$ denotes
its output jet curve and
$\Sigma( \mathbf{a} )
:= \sum_{j=1}^{k} \varrho_{j-1}\bigl( \roll( \mathbf{a} )|_{[j-1,\, j]} \bigr)$
its running cost, the partial sum of \eqref{eq:winagg}. The
\emph{complete tree} of depth $T$ generated by $\A$ is the
family
\[
  \cl^{\circ}
  := \bigl\{\, \tra_{\mathbf{a}} \;:\;
  \mathbf{a} \in \A^{k} ,\ 0 \le k \le T \,\bigr\} .
\]

\smallskip
\noindent (iii) A \emph{dispersive tree} at the pruning
radius $s$ is a subtree $\cl \subset \cl^{\circ}$ retaining
one representative node per occupied spatiotemporal cell
$( c, k ) \in s \Z^{|R|} \times \mathbb{N}$, namely the
node of the best running cost among those sharing the
cell. In particular, the tree is generated by depthwise
de-duplication from the root, $V_{0} := \{ ( ) \}$, and
\[
  \widehat{V}_{k+1}
  := \left\{ ( \mathbf{a}, a )
  \;\middle|\;
  \begin{gathered}
  \mathbf{a} \in V_{k} ,\quad a \in \A , \\
  g_{k}\bigl( \roll( ( \mathbf{a}, a ) )|_{[k,\, k+1]}
  \bigr) \ge 0 \quad \forall\, g \in \preds , \\
  \roll( ( \mathbf{a}, a ) )|_{[k,\, k+1]}
  \ \text{$2M$-smooth}
  \end{gathered}
  \right\} ,
\]
\[
  V_{k+1}
  := \left\{ \mathbf{a} \in \widehat{V}_{k+1}
  \;\middle|\;
  \begin{gathered}
  \Sigma( \mathbf{a} ) \le \Sigma( \mathbf{a}' )
  \ \ \forall\, \mathbf{a}' \in \widehat{V}_{k+1} \\
  \text{with}\
  q_{s}( \roll( \mathbf{a}' )( k{+}1 ) ) \\
  \qquad\qquad = q_{s}( \roll( \mathbf{a} )( k{+}1 ) )
  \end{gathered}
  \right\} ,
\]
with ties resolved by a fixed rule, so that exactly one
sequence per occupied terminal cell survives. Then
$\cl = \{\, \tra_{\mathbf{a}} \;:\;
\mathbf{a} \in \bigcup_{k=0}^{T} V_{k} \,\}$.
\end{definition}

While the complete tree of depth $T$ requires
$O( | \A |^{T} )$ forward computation, a dispersive tree is
grown in $O( | \A |\, T^{d+1} ( M / s )^{|R|} )$ forward
computation \cite{deshpande2026DFT}, polynomial in
the horizon $T$, and provides a near-optimality guarantee.
In particular, for the single agent formulation of
\eqref{eq:jocp-flat},
\begin{equation}
\label{eq:socp}
\begin{aligned}
  \min_{z(\cdot)} \quad
  & J\bigl( \jet z(\cdot) \bigr)
    \\[0.2em]
  \text{s.t.} \quad
  & \Psi_{u}\bigl( \jet z(t),\,
    z^{(R)}(t) \bigr) \in U
    \ \ \text{a.e.\ on } [0, T], \\[0.2em]
  & g\bigl( \jet z(\cdot) \bigr)
    \ge 0 \qquad \forall\, g \in \preds, \\[0.2em]
  & \jet z(0) = \Xi\bigl( x_{\mathrm{init}} \bigr) ,
\end{aligned}
\end{equation}
the dispersive tree is near-optimal in the following sense.

\begin{theorem}
\label{thm:single}
There are constants $c_{B}, C_{B} > 0$ such that the
following holds. Let the dispersion and the pruning radius
satisfy $\delta_{\A} + s \le c_{B}\, \varepsilon$, and let
$\rho := ( C_{B} + 1 )( \delta_{\A} + s )$.
Let $\tra = \tra^{x_{\mathrm{init}}}_{u}$ be a feasible
solution of \eqref{eq:socp} with
$u \in \mathcal{U}^{\varepsilon}_{M}( x_{\mathrm{init}} )$
and the margin
\[
  g\bigl( \jet z_{\tra} \bigr)
  \;\ge\;  L_{\preds}\, \rho
  \qquad \forall\, g \in \preds .
\]
Then the dispersive tree of depth $T$ retains a command
sequence $\mathbf{a} \in \A^{T}$, feasible for
\eqref{eq:socp}, with
\[
  J\bigl( \roll( \mathbf{a} ) \bigr)
  \;\le\;
  J\bigl( \jet z_{\tra} \bigr)
  + L_{J}\, \rho .
\]
\end{theorem}

As we shall see in the next section,
Theorem~\ref{thm:single} provides a path to a
near-optimal solution of the joint problem
\eqref{eq:jocp-flat}, viewed in the form \eqref{eq:socp}
for the team as a single system with stacked states and
controls. A joint-space dispersive tree then generated by
the product command set
$\A_{1} \times \dots \times \A_{N}$ contains a near-optimal
joint solution. In this work, we address
\eqref{eq:jocp-flat} in a decentralized manner, in which
each agent repeatedly grows trees generated by its own
command set $\A_{i}$ against what it learns from its peers,
and an anytime approach to the joint solution is set up
by exchanging command sequences, iterated over multiple
rounds.

\section{Dispersive Tree Agreements for Near-Optimal Coordination}
\label{sec:agreement}

In this section, we develop \emph{Dispersive Tree
Agreements} and the protocol that negotiates them,
WarpWeaver. WarpWeaver, as implemented, provides a
decentralized anytime approach to a near-optimal solution of
the JOCP \eqref{eq:jocp-flat}. We begin with an overview of our
mathematical results, followed by their presentation and the
protocol as implemented. We defer the proofs of these results to
\cref{app:proofs}.

\subsection{Overview}
\label{sec:overview}

Our presentation begins with an existence result in
\cref{sec:jcs}, in which we show that there exists a
near-optimal \emph{agreement} on agent-local dispersive trees,
each pruned at a resolution that keeps its growth computation
tractable. This result motivates the development of our
decentralized protocol. We begin this development in
\cref{sec:ilp} by restricting the problem of finding an optimal
agreement to an integer linear program (ILP), over the command
sequence sets $\A_{i}^{T}$ of all agents, at a chosen
resolution of space and time.

Subsequently, in \cref{sec:pricing}, we develop our solution
scheme for the ILP based on column generation. We show that the
search for a near-optimal agreement can be decomposed into an
iterative loop that involves solving the linear relaxation of
the ILP restricted to a small pool of the agents' command
sequences, or \emph{columns}, and a \emph{decentralized} column
generation procedure that enlarges the pool until no agent finds an improving column.
We prove that, whenever this process reaches quiescence, the agents select from the pool an agreement that is
near-optimal
relative to the solutions of the JOCP that are robust at the search resolution, up to a gap that the agents compute online.

\subsection{Existence of Near-Optimal Agreements}
\label{sec:jcs}

For each agent $i$, fix the constants $M_{i}$ and
$\varepsilon_{i}$ of \cref{def:classes}, a
$\delta_{\A_{i}}$-dispersive command set $\A_{i}$, and a pruning
radius $s_{i}$ satisfying the hypotheses of
\cref{thm:single} for its own dynamics, let
$c_{B,i}, C_{B,i}$ be the corresponding constants, and let
the \emph{search resolution} be
\begin{equation}
\label{eq:rho}
  \rho := \max_{i} \rho_{i} , \qquad
  \rho_{i} := ( C_{B,i} + 1 )( \delta_{\A_{i}} + s_{i} ) .
\end{equation}
For a command sequence
$\mathbf{a}^{i} \in \A_{i}^{T}$, let
$\roll_{i}( \mathbf{a}^{i} )$ denote the output jet curve
of agent $i$ generated from $x^{i}_{\mathrm{init}}$ as in
\cref{def:mincost}(ii). Restricting the JOCP
\eqref{eq:jocp-flat} to the joint command space
$\prod_{i} \A_{i}^{T}$ and expanding the joint cost by
\eqref{eq:aggregator} yields the \emph{joint command problem}
(JCP)
\begin{equation}
\label{eq:jcp}
\begin{aligned}
  \min_{\mathbf{a}^{1:N}} \
  & \max_{1 \le \ell \le L}\ \sum_{i} \alpha_{\ell i}\,
    J_{i}\bigl( \roll_{i}( \mathbf{a}^{i} ) \bigr)
    \\[0.2em]
  \text{s.t.} \
  & g\bigl( \roll_{i}( \mathbf{a}^{i} ) \bigr) \ge 0
    \qquad \forall\, g \in \preds_{i} , \\[0.2em]
  & g\bigl( ( \roll_{j}( \mathbf{a}^{j} ) )_{j \in e}
    \bigr) \ge 0
    \qquad \forall\, g \in \preds_{e} ,
    \ \forall\, e , \\[0.2em]
  & \mathbf{a}^{i} \in \A_{i}^{T} ,
    \qquad i = 1, \dots, N .
\end{aligned}
\end{equation}
We call a feasible solution of the JCP an \emph{agreement},
and show that agent-local dispersive trees contain a
near-optimal agreement in the following sense.

\begin{proposition}
\label{prop:jcs}
Let $z^{1}(\cdot), \dots, z^{N}(\cdot)$ be a
feasible solution of \eqref{eq:jocp-flat} generated by
controls $u^{i} \in \mathcal{U}^{\varepsilon_{i}}_{M_{i}}(
x^{i}_{\mathrm{init}} )$ with the margins
\begin{equation}
\label{eq:refmargin}
\begin{gathered}
  g\bigl( \jet z^{i} \bigr) \ge L_{\preds}\, \rho
  \qquad \forall\, g \in \preds_{i} ,\ \forall\, i , \\
  g\bigl( ( \jet z^{j} )_{j \in e} \bigr) \ge L_{\preds}\,
  \rho
  \qquad \forall\, g \in \preds_{e} ,\ \forall\, e .
\end{gathered}
\end{equation}
Then there exist families $\preds'_{i} \supseteq \preds_{i}$
of individual predicates of the form \eqref{eq:winsep} such
that the product $\prod_{i} \cl_{i}$ of the dispersive trees
$\cl_{i}$ generated by $\A_{i}$ at the pruning radii $s_{i}$
under the predicates $\preds'_{i}$ contains an agreement
$( \mathbf{a}^{1}, \dots, \mathbf{a}^{N} )$ with
\[
\begin{aligned}
  &\Phi\bigl( \roll_{1}( \mathbf{a}^{1} ), \dots,
  \roll_{N}( \mathbf{a}^{N} ) \bigr) \\
  &\qquad \le\;
  \Phi\bigl( \jet z^{1}, \dots, \jet z^{N} \bigr)
  + L_{\Phi}\, \rho ,
\end{aligned}
\]
where $L_{\Phi} := \max_{\ell} \sum_{i} \alpha_{\ell i}\, L_{J}$.
\end{proposition}

\Cref{prop:jcs} thus guarantees a priori that restricting the
JOCP to agreements loses at most $L_{\Phi}\, \rho$ relative to
the solutions with the margins \eqref{eq:refmargin}, and that
agent-local dispersive trees suffice to represent a
near-optimal agreement. Note that the predicates $\preds'_{i}$ confine each agent $i$
to a corridor of radius $\rho$ around $z^{i}(\cdot)$
(\cref{app:jcs}), and thus replace the joint predicates by
individual ones, under which the agents search independently. As
these predicates depend on the solution itself, \cref{prop:jcs}
is an existence result. The remainder of this section develops a
distributed computing scheme that pursues this structure,
coordinating the independent searches of the agents through a
negotiation.

\subsection{Discretized Restriction of the Agreement Problem}
\label{sec:ilp}

We begin by restricting the JCP \eqref{eq:jcp} to an ILP over
the command sequences of all agents. To this end, we partition
the jet space of each agent into cells of a chosen diameter and
the time horizon into substeps of a chosen length, and enforce each joint predicate by forbidding the
members of its group to jointly occupy a tuple of cells on
which the predicate fails.

Let $M := \max_{i} M_{i}$. Fix a cell size $S > 0$ and a
substep $h := 1 / n_{s}$, $n_{s} \in \mathbb{N}$. Denote
the timestamp $t_{m} := m h$ and the interval
$I_{m} := [t_{m}, t_{m+1}]$ for $m = 0, \dots, T n_{s} - 1$,
and fix a partition $\cells_{i}$ of the jet space
$\R^{|R_{i}|}$ of each agent $i$ into cells of diameter at
most $S$. A jet
curve $\zeta$ of agent $i$ occupies at the substep $m$ the
collection of cells
\begin{equation}
\label{eq:occ}
  \occ_{i}( \zeta, m )
  := \bigl\{\, \kappa \in \cells_{i} \;:\;
  \zeta( t ) \in \kappa \ \text{for some}\ t \in I_{m}
  \,\bigr\} .
\end{equation}
For a joint predicate $g \in \preds_{e}$ and a substep
$m$, we enumerate the \emph{forbidden tuples} of cells, on
which $g$ fails to hold, as
\begin{equation}
\label{eq:forbidden}
  F_{m}( g )
  := \Bigl\{ ( \kappa_{j} )_{j \in e} :
  \inf_{t \in I_{m}}
  \inf_{\xi_{j} \in \kappa_{j},\, j \in e}
  h_{t}\bigl( ( \xi_{j} )_{j \in e} \bigr) < 0
  \Bigr\} .
\end{equation}
For each command sequence $c \in \A_{i}^{T}$ of agent $i$, let
the binary variable $\lambda_{i, c} \in \{ 0, 1 \}$ indicate
that agent $i$ selects $c$, and denote for concision
$J_{i}( c ) := J_{i}( \roll_{i}( c ) )$ and
$\occ_{i}( c, m ) := \occ_{i}( \roll_{i}( c ), m )$. The ILP
at the resolution $( S, h )$ reads
\begin{equation}
\label{eq:ilp}
\begin{aligned}
  & \min_{\vartheta,\, \lambda}\ \vartheta \quad \text{s.t.}
    \\[0.2em]
  & \quad \vartheta \ge \sum_{i} \alpha_{\ell i}
    \sum_{c \in \A_{i}^{T}} J_{i}( c )\, \lambda_{i, c}
    \qquad \ell = 1, \dots, L , \\[0.2em]
  & \quad \sum_{c \in \A_{i}^{T}} \lambda_{i, c} = 1
    \qquad i = 1, \dots, N , \\[0.2em]
  & \quad g\bigl( \roll_{i}( c ) \bigr)\, \lambda_{i, c} \ge 0
    \qquad \forall\, g \in \preds_{i} ,\
    \forall\, c \in \A_{i}^{T} ,\ \forall\, i , \\[0.2em]
  & \quad \sum_{j \in e} \sum_{c \in \A_{j}^{T}}
    \mathbf{1}\bigl[ \kappa_{j} \in \occ_{j}( c, m ) \bigr]\,
    \lambda_{j, c} \le |e| - 1 \\
  & \qquad\qquad
    \forall\, \kappa \in F_{m}( g ) ,\
    \forall\, g \in \preds_{e} ,\ \forall\, e ,\
    \forall\, m , \\[0.2em]
  & \quad \lambda_{i, c} \in \{ 0, 1 \} .
\end{aligned}
\end{equation}
The ILP selects one sequence per agent through the binary
variables $\lambda$, bounds the weighted sums in the maximum of
\eqref{eq:jcp} by $\vartheta$, and admits only the sequences that
obey the individual predicates. The joint predicates couple the
selections within a group and enter through the
\emph{occupancy rows}, which disallow all members of a group
$e$ to occupy their respective cells of a forbidden tuple.
Relative to the JCP \eqref{eq:jcp}, the ILP tightens the joint
predicates by at most the \emph{discretization margin}
$\gamma_{S} := L_{\preds}\, ( S + 2 M h )$, as formalized by
the proposition below.

\begin{proposition}
\label{prop:ilp}
\noindent (i) Let $( \vartheta, \lambda )$ be a feasible
solution of \eqref{eq:ilp}. Then the joint command
sequence $( c_{1}, \dots, c_{N} )$ with
$\lambda_{i, c_{i}} = 1$ is an agreement with
$\Phi( \roll_{1}( c_{1} ), \dots, \roll_{N}( c_{N} ) )
\le \vartheta$.

\noindent (ii) Let
$( \mathbf{a}^{1}, \dots, \mathbf{a}^{N} )$ be an agreement
whose rollouts
$\roll_{i}( \mathbf{a}^{i} )$ are $2 M_{i}$-smooth, with
the margins
\[
  g\bigl( ( \roll_{j}( \mathbf{a}^{j} ) )_{j \in e} \bigr)
  \ge \gamma_{S}
  \qquad \forall\, g \in \preds_{e} ,\ \forall\, e .
\]
Then $( \vartheta, \lambda )$ with
\[
\begin{gathered}
  \lambda_{i, \mathbf{a}^{i}} = 1 , \qquad i = 1, \dots, N , \\
  \vartheta = \Phi\bigl( \roll_{1}( \mathbf{a}^{1} ), \dots,
  \roll_{N}( \mathbf{a}^{N} ) \bigr)
\end{gathered}
\]
is a feasible solution of \eqref{eq:ilp}, and hence the
optimal value of \eqref{eq:ilp} is at most $\vartheta$.
\end{proposition}

\subsection{Decentralized Column Generation}
\label{sec:pricing}

As formulated, the ILP \eqref{eq:ilp} holds a binary decision
variable $\lambda_{i, c}$ for each command sequence $c$ of each
agent $i$, and is thus of enormous size. Column generation
\cite{dantzig1960decomposition, barnhart1998price,
lubbecke2005column} is the canonical method to address integer
programs with too many variables to enumerate. The procedure
involves solving the linear relaxation of the ILP over a small
pool of variables, or \emph{columns}, and enlarging the pool
with columns that improve the relaxation under prices derived
from its dual. The relaxation provides a lower bound to certify the integer
solutions that we select from the pool.

For multi-robot coordination problems with our structure, we
show that column generation can be effectively decentralized
across the agents due to the separable structure of the
corresponding ILP, in which the agents interact only through
the aggregator and the joint predicates. This allows each
agent to solve a local problem akin to the single-agent
problem \eqref{eq:socp}, priced by its interactions with its
peers, on its onboard computer. In this section, we formalize the
mathematical apparatus, summarized by \cref{fig:math-overview},
that underlies our coordination protocol described in the
subsequent section.

\begin{figure}[t]
\centering
\includegraphics[width=\columnwidth]{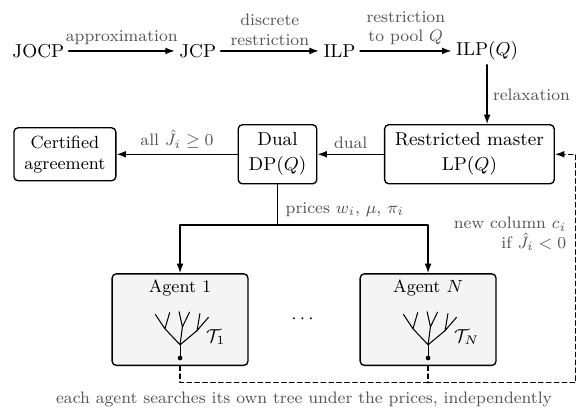}
\caption{Decentralized column generation. The dual of the restricted
master over the pool $Q$ sets common prices, under which each agent
searches its own dispersive tree. Sequences that improve the pool
return as new columns until the agreement is certified at
quiescence.}
\label{fig:math-overview}
\end{figure}

\noindent We first compress the notation of the ILP \eqref{eq:ilp}. Let
\begin{equation}
\label{eq:columns}
  \mathcal{C}
  := \bigl\{\, ( i, c ) \;:\; c \in \A_{i}^{T} ,\
  i = 1, \dots, N \,\bigr\}
\end{equation}
denote the set of \emph{columns}, one per command sequence
of each agent, so that $\lambda \in \{ 0, 1 \}^{\mathcal{C}}$.
The constraints of \eqref{eq:ilp} form four blocks of
\emph{rows}, given by the matrices over the columns
\begin{equation}
\label{eq:blocks}
\begin{aligned}
  \mathbf{J}_{\ell, ( i, c )}
  &:= \alpha_{\ell i}\, J_{i}( c ) ,
  \qquad \ell = 1, \dots, L , \\
  \mathbf{E}_{i', ( i, c )}
  &:= \mathbf{1}[\, i' = i \,] ,
  \qquad i' = 1, \dots, N , \\
  \mathbf{G}_{( i', g, c' ), ( i, c )}
  &:= g\bigl( \roll_{i}( c ) \bigr)\,
  \mathbf{1}[\, ( i', c' ) = ( i, c ) \,] , \\
  &\qquad\qquad ( i', c' ) \in \mathcal{C} ,\
  g \in \preds_{i'} , \\
  \mathbf{O}_{r, ( i, c )}
  &:= \mathbf{1}\bigl[\, i \in e ,\
  \kappa_{i} \in \occ_{i}( c, m ) \,\bigr] , \\
  &\qquad\qquad r = ( g, m, \kappa ) \in \res ,
\end{aligned}
\end{equation}
where $\res$ collects the rows $r = ( g, m, \kappa )$ with
$\kappa \in F_{m}( g )$ and $g \in \preds_{e}$ over all
groups $e$ and substeps $m$, with the right-hand sides
$b_{r} := |e| - 1$. For a \emph{pool} $Q \subseteq \mathcal{C}$
of columns, the ILP \eqref{eq:ilp} restricted to this pool,
denoted $\mathrm{ILP}( Q )$, reads
\begin{equation}
\label{eq:ilp-canon}
\begin{aligned}
  & \mathrm{ILP}( Q ) \colon\ \min_{\vartheta,\, \lambda}\ \vartheta \quad \text{s.t.}
    \\[0.2em]
  & \ \vartheta \mathbf{1} \ge \mathbf{J} \lambda , \quad
    \mathbf{E} \lambda = \mathbf{1} , \quad
    \mathbf{G} \lambda \ge 0 , \\[0.2em]
  & \ \mathbf{O} \lambda \le b , \quad
    \lambda \in \{ 0, 1 \}^{\mathcal{C}} , \quad
    \lambda_{( i, c )} = 0
    \ \ \forall\, ( i, c ) \notin Q ,
\end{aligned}
\end{equation}
with $Q = \mathcal{C}$ recovering \eqref{eq:ilp} exactly.
Likewise, $\mathrm{JCP}( Q )$ denotes the JCP \eqref{eq:jcp} with each
$\mathbf{a}^{i}$ restricted to the sequences $c$ with
$( i, c ) \in Q$.

The \emph{restricted master problem}, denoted
$\mathrm{LP}( Q )$, relaxes $\mathrm{ILP}( Q )$ in two respects. It
lets each agent select its pooled command sequences fractionally,
with the relaxed $\lambda \ge 0$, and lets the members of each group
$e$ occupy a forbidden tuple $\kappa \in F_{m}( g )$ of a joint
predicate $g \in \preds_{e}$ at a price $P > 0$ for a slack
$s_{r} \ge 0$ on its row $r = ( g, m, \kappa )$,
\begin{equation}
\label{eq:master}
\begin{aligned}
  & \mathrm{LP}( Q ) \colon\ \min_{\vartheta,\, \lambda,\, s}\ \vartheta
    {} + P\, \mathbf{1}^{\top} s \quad \text{s.t.}
    \\[0.2em]
  & \ \vartheta \mathbf{1} \ge \mathbf{J} \lambda , \quad
    \mathbf{E} \lambda = \mathbf{1} , \quad
    \mathbf{G} \lambda \ge 0 , \\[0.2em]
  & \ \mathbf{O} \lambda \le b {} + s , \quad
    \lambda \in \R^{\mathcal{C}}_{\ge 0} , \quad
    s \in \R^{\res}_{\ge 0} , \\[0.2em]
  & \ \lambda_{( i, c )} = 0
    \ \ \forall\, ( i, c ) \notin Q .
\end{aligned}
\end{equation}
The slack variable priced at $P$ per unit slack keeps the
restricted master feasible for every $Q$ that holds a column for
each agent. This procedure for initializing column generation is
standard in the literature~\cite[Sec.~4.1.1]{lubbecke2005column}.
With $s = 0$, every feasible solution of $\mathrm{ILP}( Q )$ is
feasible for $\mathrm{LP}( Q )$ at the same cost, so that
$\mathrm{LP}( Q ) \le \mathrm{ILP}( Q )$ holds.
Its dual, denoted $\mathrm{DP}( Q )$, in the variables
$\beta \in \R^{L}$ of the aggregator rows
$\vartheta \mathbf{1} \ge \mathbf{J} \lambda$, $\pi \in \R^{N}$ of
the selection rows $\mathbf{E} \lambda = \mathbf{1}$, $\nu$ of
the individual predicate rows $\mathbf{G} \lambda \ge 0$, and
$\mu \in \R^{\res}$ of the occupancy rows
$\mathbf{O} \lambda \le b {} + s$, is then
\begin{equation}
\label{eq:dual}
\begin{aligned}
  & \mathrm{DP}( Q ) \colon\ \max_{\beta,\, \pi,\, \nu,\, \mu}\ 
    \mathbf{1}^{\top} \pi - b^{\top} \mu
    \quad \text{s.t.} \\[0.2em]
  & \ \bigl( \mathbf{J}^{\top} \beta
    - \mathbf{E}^{\top} \pi
    - \mathbf{G}^{\top} \nu
    + \mathbf{O}^{\top} \mu \bigr)_{( i, c )} \ge 0
    \ \ \forall\, ( i, c ) \in Q , \\[0.2em]
  & \ \mathbf{1}^{\top} \beta = 1 , \quad
    \beta \ge 0 , \quad \nu \ge 0 , \quad \ 0 \le \mu \le P\, \mathbf{1} .
\end{aligned}
\end{equation}
The dual variable $\mu_{r}$ corresponding to each occupancy row $r$
is therefore capped at $P$. Following the feasibility of the
restricted master, $\mathrm{DP}( Q )$ has an optimal solution, with
$\mathrm{DP}( Q ) = \mathrm{LP}( Q )$ by strong duality.
We henceforth, with some overload of notation, use
$\mathrm{JCP}( Q )$, $\mathrm{ILP}( Q )$, $\mathrm{LP}( Q )$, and
$\mathrm{DP}( Q )$ to also denote the optimal values of the
respective programs.
For a feasible solution $( \beta, \pi, \nu, \mu )$ of
$\mathrm{DP}( Q )$, define the \emph{reduced cost} of the
command sequence $c \in \A_{i}^{T}$ of agent $i$ as
\begin{equation}
\label{eq:reduced}
  \bar{J}_{i}( c )
  := \bigl( \mathbf{J}^{\top} \beta - \mathbf{E}^{\top} \pi
  + \mathbf{O}^{\top} \mu \bigr)_{( i, c )} .
\end{equation}
Note that we have omitted the term
$- \mathbf{G}^{\top} \nu$ corresponding to the individual
predicates since we only pool and price sequences that
satisfy them. In particular, the rows
$g( \roll_{i}( c ) )\, \lambda_{i, c} \ge 0$ of such a sequence
$c$ hold for every $\lambda_{i, c} \ge 0$, and thus never
constrain $\mathrm{LP}( Q )$, so that their duals $\nu$ can be
taken to be zero.

The reduced cost of a sequence of agent $i$ depends on its
peers only through the solution
$( \beta, \pi, \nu, \mu )$ of the dual $\mathrm{DP}( Q )$.
Indeed, the reduced cost \eqref{eq:reduced} of a sequence $c$
of agent $i$ can be expanded as
\begin{equation}
\label{eq:reduced-rent}
  \bar{J}_{i}( c )
  = w_{i}\, J_{i}( c )
  + \sum_{k=0}^{T-1} r_{i, k}\bigl( \roll_{i}( c ) \bigr)
  - \pi_{i} ,
\end{equation}
where $w_{i} := \sum_{\ell} \beta_{\ell}\, \alpha_{\ell i}$ is
the weight of agent $i$ in the aggregator under the dual, and
$r_{i, k}( \zeta )$ denotes the rent that a jet curve $\zeta$
of agent $i$ pays over the window $[k, k+1]$,
\begin{equation}
\label{eq:rent-window}
  r_{i, k}( \zeta )
  := \sum_{m \,:\, I_{m} \subseteq [k,\, k+1]}\
  \sum_{\kappa \in \occ_{i}( \zeta, m )} \mu_{i, \kappa, m} ,
\end{equation}
for the \emph{rent}
\begin{equation}
\label{eq:rent}
  \mu_{i, \kappa, m}
  := \sum_{\substack{r = ( g, m, \kappa' ) \in \res \\
  i \in e,\ \kappa'_{i} = \kappa}} \mu_{r}
\end{equation}
that agent $i$ pays for occupying the cell $\kappa$ at the
substep $m$. The rent sums the duals of the occupancy rows that
assign the cell $\kappa$ to agent $i$. In linear programming,
an optimal dual variable measures the marginal change of the
optimal value per unit increase of the right-hand side of its
row, and thus the value of the resource that the row constrains
\cite{bertsimas1997introduction}. Accordingly, we refer to the
weight $w_{i}$, the rents $\mu_{i, \kappa, m}$, and $\pi_{i}$
as the \emph{prices} faced by agent $i$. The rents price the
conflicts that an occupancy may cause with the peers, and the
reduced cost \eqref{eq:reduced-rent} evaluates a sequence of
agent $i$ at these prices.

The goal of column generation is a pool $Q$ whose restricted
master attains the relaxation over all columns, that is,
$\mathrm{LP}( Q ) = \mathrm{LP}( \mathcal{C} )$. Such a pool
$Q \subseteq \mathcal{C}$ may fall short of this requirement
only if columns outside it, in $\mathcal{C} \setminus Q$, can
lower $\mathrm{LP}( Q )$. By a standard result of column
generation \cite[Sec.~2.3]{lubbecke2005column}, this requires
the existence of a column $( i, c ) \in \mathcal{C} \setminus Q$
with a negative reduced cost $\bar{J}_{i}( c ) < 0$ under every
optimal solution of $\mathrm{DP}( Q )$. The pool can thus be
tested through the \emph{pricing problem} of each agent $i$,
which minimizes the reduced cost over its sequences that obey
the individual predicates,
\begin{equation}
\label{eq:pricing}
  \bar{J}_{i}^{*}
  := \min \bigl\{ \bar{J}_{i}( c ) : c \in \A_{i}^{T},\
  g( \roll_{i}( c ) ) \ge 0
  \ \forall g \in \preds_{i} \bigr\} .
\end{equation}

This test follows from a relationship between the dual
objective $\mathbf{1}^{\top} \pi - b^{\top} \mu$ of
$\mathrm{DP}( Q )$, the optimal values $\bar{J}_{i}^{*}$ of the
pricing problems \eqref{eq:pricing}, and
$\mathrm{LP}( \mathcal{C} )$. Ultimately, however, we measure
the pool against the solutions of the JOCP \eqref{eq:jocp-flat},
which are jet curves that command sequences need not generate.
The dual objective is similarly related to the joint cost of
such curves $\zeta^{1}, \dots, \zeta^{N}$, provided that they
obey the
occupancy rows of the ILP \eqref{eq:ilp-canon}, that is,
\emph{occupy no forbidden tuple} of cells,
\begin{equation}
\label{eq:noforbid}
  \sum_{j \in e}
  \mathbf{1}\bigl[ \kappa_{j} \in \occ_{j}( \zeta^{j}, m ) \bigr]
  \le | e | - 1
  \quad \forall ( g, m, \kappa ) \in \res .
\end{equation}
The reduced cost can be similarly extended to a jet curve
$\zeta$ of agent $i$ as
\begin{equation}
\label{eq:reduced-curve}
  \bar{J}_{i}( \zeta )
  := w_{i}\, J_{i}( \zeta )
  + \sum_{k=0}^{T-1} r_{i, k}( \zeta )
  - \pi_{i} ,
\end{equation}
so that, by \eqref{eq:reduced-rent},
$\bar{J}_{i}( \roll_{i}( c ) ) = \bar{J}_{i}( c )$ for every
command sequence $c \in \A_{i}^{T}$. We state the bounds and
the pool sufficiency test in the following proposition.

\begin{proposition}
\label{prop:lagrangian}
Let $( \beta, \pi, \nu, \mu )$ be a feasible solution of
$\mathrm{DP}( Q )$.

\noindent (i) If the jet curves
$\zeta^{1}, \dots, \zeta^{N}$ of the agents occupy no
forbidden tuple, then
\[
  \mathbf{1}^{\top} \pi - b^{\top} \mu
  + \sum_{i} \bar{J}_{i}( \zeta^{i} )
  \;\le\; \Phi( \zeta^{1}, \dots, \zeta^{N} ) .
\]

\noindent (ii) If a locally feasible command
sequence $c_{i}$ of each agent $i$ attains the optimal value
$\bar{J}_{i}( c_{i} ) = \bar{J}_{i}^{*}$ of the pricing problem
\eqref{eq:pricing}, then
\[
  \mathbf{1}^{\top} \pi - b^{\top} \mu
  + \sum_{i} \bar{J}_{i}( c_{i} )
  \;\le\; \mathrm{LP}( \mathcal{C} )
  \;\le\; \mathrm{LP}( Q ) .
\]

\noindent (iii) If the dual solution
$( \beta, \pi, \nu, \mu )$ is optimal, i.e.,
$\mathrm{DP}( Q ) = \mathbf{1}^{\top} \pi - b^{\top} \mu$,
and $\bar{J}_{i}^{*} = 0$ for all $i$, then
$\mathrm{LP}( Q ) = \mathrm{LP}( \mathcal{C} )$.
\end{proposition}

\Cref{prop:lagrangian}(iii) thus suggests that each agent must
compute optimal command sequences under prices
$( w_{i}, \mu_{i, \kappa, m}, \pi_{i} )$ set by the current
pool $Q$ to either certify its sufficiency or enlarge it
appropriately. We now show that this computation can be
formulated in the structure of a single-agent problem
\eqref{eq:socp} and that a dispersive tree pruned at some
resolution $s_{i}$ under a specific running cost can search for
a near-optimal command sequence in polynomial time. Note that
the reduced cost is additive over the windows,
so that the \emph{priced running cost}
\begin{equation}
\label{eq:priced-running}
  \Sigma^{\mu}_{i}( \mathbf{a} )
  := w_{i}\, \Sigma_{i}( \mathbf{a} )
  + \sum_{j=1}^{k} r_{i, j-1}\bigl( \roll_{i}( \mathbf{a} ) \bigr)
  - \pi_{i} ,
  \qquad \mathbf{a} \in \A_{i}^{k}
\end{equation}
is consistent with the running cost of \cref{def:mincost}.
At depth $T$, $\Sigma^{\mu}_{i}( \mathbf{a} )
= \bar{J}_{i}( \mathbf{a} )$, thus reproducing the reduced cost.
The following proposition extends \cref{thm:single} to show
that local dispersive tree search can solve \eqref{eq:pricing}
near-optimally.

\begin{proposition}
\label{prop:priced}
Suppose that $( \beta, \pi, \nu, \mu )$ is a feasible
solution of the dual $\mathrm{DP}( Q )$, and that
$z^{i}(\cdot)$ is a feasible solution of \eqref{eq:socp} for
agent $i$, generated by a control
$u^{i} \in \mathcal{U}^{\varepsilon_{i}}_{M_{i}}( x^{i}_{\mathrm{init}} )$
with the margin
\[
  g\bigl( \jet z^{i} \bigr) \;\ge\; L_{\preds}\, \rho
  \qquad \forall\, g \in \preds_{i} ,
\]
at the search resolution $\rho$ of \eqref{eq:rho}.
For the dispersive tree generated by $\A_{i}$ at the pruning
radius $s_{i}$ with the priced running cost
\eqref{eq:priced-running}, let $\hat{J}_{i}$ denote the reduced
cost of its best command sequence at depth $T$.
Then $\bar{J}_{i}^{*} \le \hat{J}_{i}$, and the tree retains
a command sequence
$c_{i} \in \A_{i}^{T}$, feasible for \eqref{eq:socp}, with
\[
\begin{aligned}
  \hat{J}_{i} \;\le\; \bar{J}_{i}( c_{i} )
  \;\le\;\; & w_{i} \bigl( J_{i}\bigl( \jet z^{i} \bigr)
  + L_{J}\, \rho \bigr) \\
  & + \sum_{m}
  \sum_{\kappa \in \occ^{\rho}_{i}( \jet z^{i}, m )}
  \mu_{i, \kappa, m}
  \; - \; \pi_{i} ,
\end{aligned}
\]
where $\occ^{\rho}_{i}$ inflates the occupancy \eqref{eq:occ}
by $\rho$,
\[
\begin{aligned}
  \occ^{\rho}_{i}( \zeta, m )
  := \bigl\{\, \kappa \in \cells_{i} \;:\;
  & d\bigl( \xi, \zeta( t ) \bigr) \le \rho \\
  & \textup{for some}\ \xi \in \kappa ,\ t \in I_{m}
  \,\bigr\} .
\end{aligned}
\]
\end{proposition}

\Cref{prop:priced} thus shows that the pricing problem
\eqref{eq:pricing} of each agent $i$ can be solved
near-optimally, with the value $\hat{J}_{i}$ approximating its
optimal value $\bar{J}_{i}^{*}$.  While the test
$\bar{J}_{i}^{*} = 0$ of \cref{prop:lagrangian}(iii) on the complete
tree $\A_{i}^{T}$ is intractable, the theorem below shows that the
values $\hat{J}_{i}$ of the forward trees pruned at the search
resolution $\rho$ suffice to certify how close the
pool $Q$ comes to the solutions of the JOCP 
\eqref{eq:jocp-flat} that are robust at the search resolution $\rho$.

\begin{theorem}
\label{thm:certificate}
Suppose that $( \beta, \pi, \nu, \mu )$ is a feasible
solution of the dual $\mathrm{DP}( Q )$, and that
$z^{1}(\cdot), \dots, z^{N}(\cdot)$ is a feasible solution of
\eqref{eq:jocp-flat}, generated by controls
$u^{i} \in \mathcal{U}^{\varepsilon_{i}}_{M_{i}}( x^{i}_{\mathrm{init}} )$,
with the margins
\[
\begin{gathered}
  g\bigl( \jet z^{i} \bigr) \ge L_{\preds}\, \rho
  \qquad \forall\, g \in \preds_{i} ,\ \forall\, i , \\
  g\bigl( ( \jet z^{j} )_{j \in e} \bigr)
  \ge L_{\preds}\, ( \rho + S + M h )
  \qquad \forall\, g \in \preds_{e} ,\ \forall\, e ,
\end{gathered}
\]
at the search resolution $\rho$ of \eqref{eq:rho}.
For each agent $i$, let $\hat{J}_{i}$ denote the reduced
cost of the best command sequence at depth $T$ of the dispersive
tree generated by $\A_{i}$ at the pruning radius $s_{i}$ with
the priced running cost \eqref{eq:priced-running}.

\noindent (i) It holds that
\[
  \mathbf{1}^{\top} \pi - b^{\top} \mu
  + \sum_{i} \hat{J}_{i}
  \;\le\;
  \Phi\bigl( \jet z^{1}, \dots, \jet z^{N} \bigr)
  + L_{\Phi}\, \rho .
\]

\noindent (ii) If the dual solution is optimal, i.e.,
$\mathrm{DP}( Q ) = \mathbf{1}^{\top} \pi - b^{\top} \mu$, then
\[
  \mathrm{LP}( Q ) + \sum_{i} \hat{J}_{i}
  \;\le\;
  \Phi\bigl( \jet z^{1}, \dots, \jet z^{N} \bigr)
  + L_{\Phi}\, \rho ,
\]
and in particular
$\mathrm{LP}( Q ) \le
\Phi( \jet z^{1}, \dots, \jet z^{N} ) + L_{\Phi}\, \rho$ if
$\hat{J}_{i} \ge 0$ for all $i$.
\end{theorem}

The values $\hat{J}_{i}$ thus either exhibit columns that
enlarge the pool $Q$ or certify it. If
$\hat{J}_{i} < 0$, the command sequence of agent $i$ that
attains $\hat{J}_{i}$ has a negative reduced cost and is to be
added to the pool. If $\hat{J}_{i} \ge 0$ for all $i$,
\cref{thm:certificate}(ii) shows that the pool is certified up
to $L_{\Phi}\, \rho$ relative to the solutions of the JOCP that
are robust at the search resolution $\rho$. This result allows
us to develop a decentralized anytime protocol, presented
below, in which the agents select an agreement from the pool
with its cost certified against the JOCP.

\subsection{Protocol and Real-Time Implementation}
\label{sec:protocol}

We now present the WarpWeaver protocol, which implements decentralized
column generation among the agents as a real-time negotiation
process in a receding-horizon manner. In \cref{alg:dcg}, we
state the negotiation from the perspective of agent $i$, and
\cref{fig:protocol} depicts the same. Each agent runs an
identical loop onboard and communicates by broadcasting
generated columns as well as endorsing joint solutions over the
joint pool. Each agent broadcasts to all others, so that every agent holds
the same pool and solves the same restricted master.

In \cref{alg:dcg}, a joint solution is \emph{globally endorsed}
once every agent has published or endorsed it, and the best
such solution has the least cost $\Phi$, with ties broken by a
common rule. An iteration is \emph{quiescent} once all agents
hold the same pool $Q$, compute $\hat{J}_{i} \ge 0$ under the
common optimal solution of $\mathrm{DP}( Q )$, and endorse a
common solution of $\mathrm{JCP}( Q )$. Since quiescence cannot
be guaranteed within the real-time deadlines, the agents commit at the
pick time $t_{n}$ to the best globally endorsed solution, whose cost is
certified by \cref{prop:anytime}(ii) below only if the iteration is
quiescent. An agent sends no
message after time $t_{n} - \delta_{\mathrm{comm}}$, even within a pass of the loop
begun before it. The pick margin
$\tau$ leaves time to commit before
the committed windows begin at time $( n + 1 )\, \Delta_{s}$.
\Cref{prop:anytime} formalizes the properties of
WarpWeaver stated above.

\begin{figure}[t]
\centering
\includegraphics[width=\columnwidth]{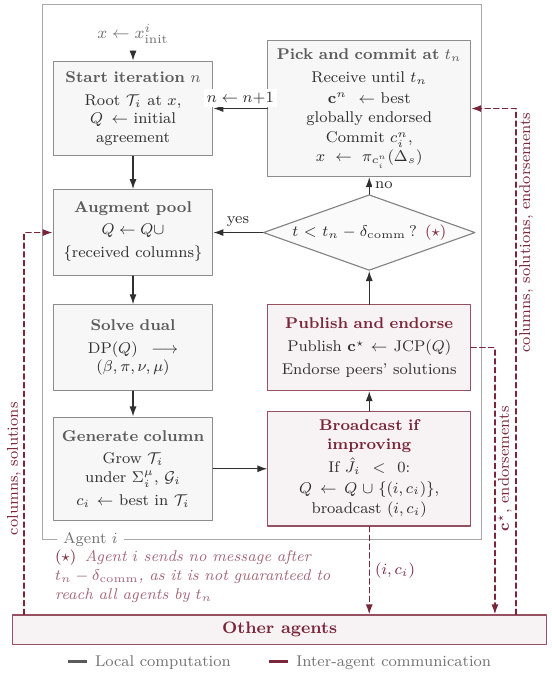}
\caption{The WarpWeaver protocol
(\cref{alg:dcg}), visualized as a flowchart at the onboard
computer of agent $i$. Each agent asynchronously runs an
identical loop onboard.}
\label{fig:protocol}
\end{figure}

\begin{algorithm}[!t]
\caption{WarpWeaver at agent $i$}
\label{alg:dcg}
\begin{algorithmic}[1]
\Require Initial state $x^{i}_{\mathrm{init}}$, command set
  $\A_{i}$, cost $J_{i}$, predicates $\preds_{i}$, stride $\Delta_{s}$,
  pick margin $\tau$, delay bound $\delta_{\mathrm{comm}}$
\State $x \gets x^{i}_{\mathrm{init}}$
\For{$n = 0, 1, 2, \dots$}
  \Comment{Pick time $t_{n} := ( n + 1 )\, \Delta_{s} - \tau$}
  \State Root the dispersive tree $\cl_{i}$ at $x$, and
    $Q \gets$ the columns of the initial agreement
  \While{$t < t_{n} - \delta_{\mathrm{comm}}$}
    \Comment{Nothing sent after $t_{n} - \delta_{\mathrm{comm}}$}
    \State $Q \gets Q \cup \{$received columns$\}$
    \State $( \beta, \pi, \nu, \mu ) \gets$ optimal solution
      of $\mathrm{DP}( Q )$
    \State Grow $\cl_{i}$ subject to $\Sigma^{\mu}_{i}$ and
      $\preds_{i}$
    \State $c_{i} \gets$ the best sequence in $\cl_{i}$
    \If{$\hat{J}_{i} < 0$}
      \State $Q \gets Q \cup \{ ( i, c_{i} ) \}$, and
        broadcast $( i, c_{i} )$
    \EndIf
    \State Publish $\mathbf{c}^{\star} \gets$  best solution of
      $\mathrm{JCP}( Q )$ found, if any
    \State Endorse peers' solutions with columns in $Q$
  \EndWhile
  \State Receive until $t_{n}$
  \State $\mathbf{c}^{n} \gets$ best globally endorsed solution
  \State Commit $c^{n}_{i}$ on $[ ( n + 1 ) \Delta_{s}, ( n + 2 ) \Delta_{s} ]$
  \State $x \gets \tra_{c^{n}_{i}}( \Delta_{s} )$
\EndFor
\end{algorithmic}
\end{algorithm}

\begin{proposition}
\label{prop:anytime}
Suppose that the following hold.

\smallskip
\noindent (a) \emph{Bounded delay.} Every message broadcast by
an agent reaches every other agent within time $\delta_{\mathrm{comm}}$.

\smallskip
\noindent (b) \emph{Common dual.} Agents that share a pool
compute the same optimal solution of $\mathrm{DP}( Q )$.

\smallskip
\noindent (c) \emph{Recursive feasibility.} The pool of every
iteration is initialized with an agreement endorsed by every
agent.

\smallskip
\noindent Then the following hold.

\smallskip
\noindent (i) \emph{Completeness and consistency.} At every
pick time $t_{n}$, all agents commit to the same joint solution
$\mathbf{c}^{n}$.

\smallskip
\noindent (ii) \emph{Certificate at quiescence.} If iteration
$n$ is quiescent with the pool $Q$ before time
$t_{n} - \delta_{\mathrm{comm}}$,
then, for every solution $z^{1}(\cdot), \dots, z^{N}(\cdot)$ of
\eqref{eq:jocp-flat} with the margins of
\cref{thm:certificate},
\[
  \Phi( \mathbf{c}^{n} )
  - \Phi\bigl( \jet z^{1}, \dots, \jet z^{N} \bigr)
  \;\le\; \Gamma_{n} + L_{\Phi}\, \rho ,
\]
where $\Gamma_{n} := \Phi( \mathbf{c}^{n} ) - \mathrm{LP}( Q )$.

\smallskip
\noindent Therefore, the cost of the solution committed at
quiescence exceeds the $\rho$-robust optimum by no more than its excess
$\Gamma_{n}$ over the restricted master, up to the resolution loss
$L_{\Phi}\, \rho$.
\end{proposition}

\begin{remark}
(i) Our implementation ensures the common dual by a canonical
restricted master and a deterministic solver~\cite{huangfu2018highs}.
$\mathrm{JCP}( Q )$, on the other hand, is solved by depth-first
branch-and-bound over the agents. We use
$P = \max_{\ell,\, ( i, c ) \in Q} \alpha_{\ell i}\, J_{i}( c )$,
the largest weighted cost of a command sequence in the pool, to drive
the column generation.

\smallskip
\noindent (ii) Recursive feasibility and its enforcement are specific
to the JOCP at hand. In practice, if recursive feasibility is
lost, computational speed becomes essential for completeness,
so that a globally endorsed solution can be found before the
commit is due.

\smallskip
\noindent (iii)  In our simulations,
$\Gamma_{n} \le 0$ at each of the $1502$ quiescent iterations 
(\cref{tab:dynobench,tab:ego-fan,tab:scale-rt,tab:scale-tq}), so that
\cref{prop:anytime}(ii) held with the resolution loss $L_{\Phi}\, \rho$ alone. Note that $\Gamma_{n} < 0$ occurs as the joint solution
$\mathbf{c}^{n}$ is subjected to the exact joint predicates while
$\mathrm{LP}( Q )$ carries the conservative occupancy rows.
\end{remark}

\begin{table}[t]
\caption{Platforms, command sets, and search parameters. Each command set also contains a hold and, for the integrators, a brake toward rest~\cite{deshpande2026DFT}. $G_{n}[ l, u ]$: $n$-level grid on $[ l, u ]$; $e_{\theta} = ( \cos\theta, \sin\theta )$. Triple integrator at $x = ( p, v, a )$: $j^{1}_{\bar v} = -a / \Delta \pm 4 \lambda_{1} ( \bar v - v - \Delta a / 2 ) / \Delta^{2}$, with $+$ on the first half window, and $j^{2}_{\bar v} = 2 \lambda_{2} ( \bar v - v - \Delta a ) / \Delta^{2}$, with $\lambda_{1}, \lambda_{2} \le 1$ maximal within the bounds; $V( x ) = \{ \sigma ( \cos\phi\, e_{\psi + \theta}, \sin\phi ) \}$, $\sigma \in \{ \frac{4}{3}, \frac{8}{3}, 4 \}$, $\theta \in \{ 0, \pm 5^{\circ}, \pm 10^{\circ}, \pm 20^{\circ} \}$, $\phi \in \{ 0, \pm 20^{\circ} \}$, $\psi$ the heading of $v$. X500: attitude $q$, body axis $b_{3}$, body rate $\omega$, rotor thrusts $\eta$ (N), mixer $B$, motor time constant $t_{m}$. $\Delta$: window length; $T$: horizon in windows; $\Delta_{s}$: stride; $s_{i}$: pruning radius; $\varepsilon_{g}$: goal tolerance; $S$: cell diameter; $n_{s}$: substeps per window.}
\label{tab:platforms}
\centering
\footnotesize
\renewcommand{\arraystretch}{1.2}
\setlength{\tabcolsep}{2.5pt}
\begin{tabularx}{\columnwidth}{|>{\raggedright\arraybackslash}p{1.05cm}|>{\raggedright\arraybackslash}X|>{\raggedright\arraybackslash}p{1.85cm}|>{\raggedright\arraybackslash}p{2.75cm}|}
\hline
 & $f_{i}$, $U_{i}$ & $E_{i}$; $R_{i}$; $B_{i}$ (m) & $\A$ \\
\hline
Unicycle
  & $\dot{p} = v e_{\theta}$, $\dot{\theta} = \omega$, $|v|, |\omega| \le 0.5$
  & --; --; Box $0.5 {\times} 0.25$
  & \emph{Constant $( v, \omega )$ on a lattice}\newline $( v, \omega ) \in G_{8}\bigl[ -\tfrac{1}{2}, \tfrac{1}{2} \bigr]^{2}$ \\
\hline
Double Int.\ (3D)
  & $\ddot{p} = a$, $\| a \|_{\infty} \le 2$
  & $\| v \|_{\infty} \le 0.5$; $\| v \|_{\infty} \le 0.17$; Ball $0.1$
  & \emph{Constant $a$ on a lattice of the admissible box}\newline $a \in G_{4}\bigl( J( x ) \bigr)$, with $J( x ) = [ -2, 2 ]^{3} \cap \bigl( [ -\tfrac{1}{2}, \tfrac{1}{2} ]^{3} - v \bigr) / \Delta$ \\
\hline
Triple Int.\ (3D)
  & $\dddot{p} = j$, $\| j \| \le 25$
  & $\| v \| \le 4$, $\| a \| \le 8$; $\| v \|_{\infty} \le 0.3$, $\| a \|_{\infty} \le 0.6$; Ball $0.15$
  & \emph{Admissible jerk profiles to diverse reference velocities}\newline $j \in \bigl\{ j^{1}_{\bar v},\, j^{2}_{\bar v} : \bar v \in V( x ) \bigr\}$ \\
\hline
X500
  & $\dot p = v$, $m \dot v = c\, b_{3}( q ) - m g e_{3}$, $\dot q = \tfrac{1}{2} q \otimes \omega$, $J \dot \omega = M - \omega \times J \omega$, $( c, M ) = B \eta$, $t_{m} \dot \eta = u - \eta$, $u \in [ 0.19, 8.55 ]^{4}$
  & $\| v \| \le 4$, $| \omega_{1,2} | \le 3.84$, $| \omega_{3} | \le 3.49$, $\angle( b_{3}, e_{3} ) \le 45^{\circ}$; $\| v \|_{\infty} \le 0.5$, $\| \omega \|_{\infty} \le 0.5$; Ball $0.1$
  & \emph{Smooth acceleration references on a lattice, tracked by the geometric controller of~\cite{deshpande2026DFT}}\newline $a \in G_{5}[ -8, 8 ]^{2} \times \{ -8, -4, 0, 1.85, 3.7 \}$,\newline $c_{-} \le \| a + g e_{3} \| \le c_{+}$ \\
\hline
Single Int.\ (2D)
  & $\dot{p} = u$, $\| u \| \le 1.5$
  & --; --; Disk $0.36$
  & \emph{Constant $u$ on a polar lattice}\newline $u \in \{ \sigma e_{2 \pi k / 16} : \sigma \in \{ 0.5, 1, 1.5 \},$ $k < 16 \}$ \\
\hline
\end{tabularx}

\smallskip
\setlength{\tabcolsep}{2.5pt}
\begin{tabularx}{\columnwidth}{|l|*{7}{>{\centering\arraybackslash}X|}}
\hline
Search & $\Delta$ (s) & $T$ & $\Delta_{s}$ (s) & $s_{i}$ & $\varepsilon_{g}$ & $S$ (m) & $n_{s}$ \\
\hline
Unicycle & $0.5$ & $10$ & $0.5$ & $0.08$ & $0.5$ & $0.6$ & $5$ \\
Double Int. & $0.5$ & $10$ & $0.5$ & $0.25$ & $0.5$ & $0.25$ & $5$ \\
Triple Int. & $0.5$ & $8$ & $0.5$ & $2$ & $0.75$ & $0.35$ & $10$ \\
X500 & $0.25$ & $6$ & $0.5$ & $2$ & $0.75$ & $0.6$ & $25$ \\
Single Int. & $1$ & $10$ & $1$ & $0.1$ & $1.5$ & $0.75$ & $10$ \\
\hline
\end{tabularx}
\end{table}

\begin{figure*}[t]
\centering
\newcommand{\mpanel}[3]{\sbox0{\includegraphics[height=#2]{#1}}%
  \parbox[b]{\wd0}{\centering\usebox0\\[1pt]{\footnotesize (#3)}}}
\mpanel{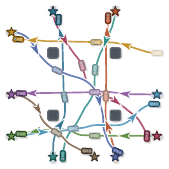}{2.85cm}{a}\hfill
\mpanel{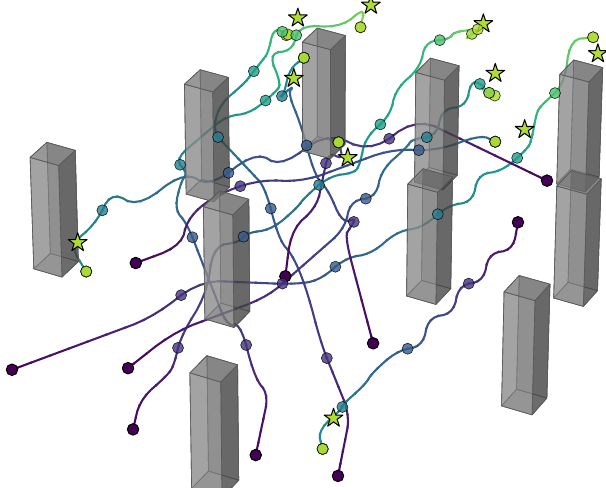}{2.85cm}{b}\hfill
\mpanel{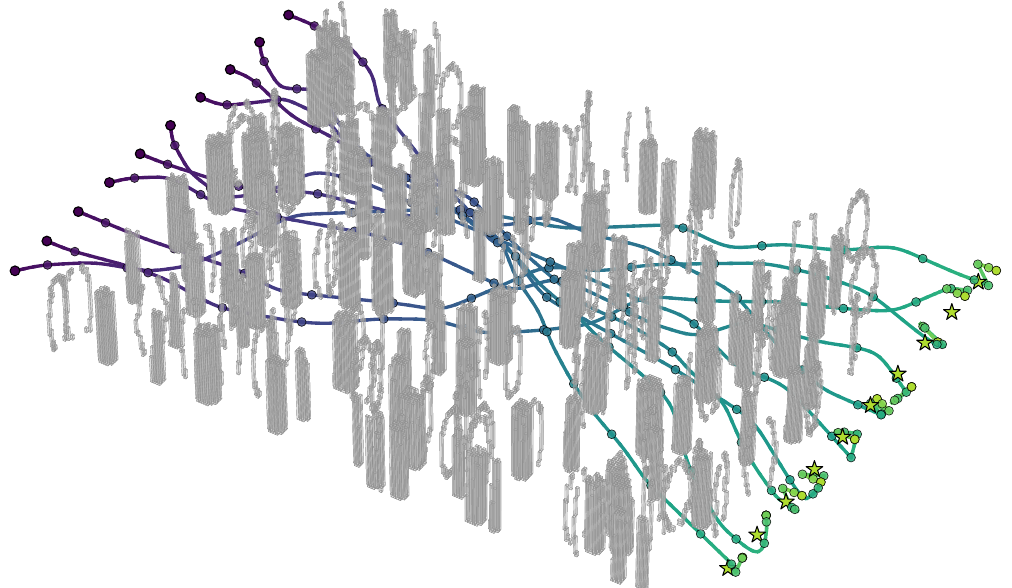}{2.85cm}{c}\hfill
\mpanel{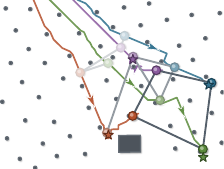}{2.85cm}{d}
\caption{WarpWeaver in simulation. (a) Circle-10, unicycles. (b) Forest-10, 3D double integrators. (c) Ten triple integrators on a forest map of EGO-Swarm. (d) Connectivity maintenance in the forest world of Realm, with the live links at three instants. In (a) and (d), each robot has its own colour and is drawn at three instants, later ones more opaque. In (b) and (c), trajectories and bodies are coloured by time from start (circle) to goal (star).}
\label{fig:sim}
\end{figure*}

\section{Simulation Experiments and Benchmarks}
\label{sec:sim}

We evaluate WarpWeaver in simulation in three settings (\cref{fig:sim}). First, we
compare its cost against state-of-the-art offline kinodynamic
planners on the multi-robot Dynobench suite. Second, we compare it against a
decentralized real-time planner on swarms of ten drones, including a
high-fidelity nonlinear quadrotor model.
We also evaluate WarpWeaver on a task that requires maintaining
connectivity of the communication graph across the team, and
compare its performance against a dedicated planner. Unless stated
otherwise, each robot runs \cref{alg:dcg} in real time on its own
share of a GPU, eight streaming multiprocessors of an RTX 3090,
equivalent in compute to an Orin Nano. We set the pick margin to
$\tau = 20$\,ms and the communication delay bound to
$\delta_{\mathrm{comm}} = 30$\,ms. Each robot runs as its own process and
exchanges its messages with the others through inter-process
queues, with real delays. In two trials of the largest instance on
each platform of \cref{tab:dynobench,tab:ego-fan}, the $358{,}550$
messages arrived after a median delay of $0.7$\,ms, $99.9\%$ of them
within $10$\,ms and all within $26$\,ms, under the delay bound
$\delta_{\mathrm{comm}}$, and the robots committed the same agreement at all $264$
picks.

\subsection{Multi-Robot Kinodynamic Motion Planning}
\label{sec:sim-kino}

\begingroup\predisplaypenalty=0
A major portion of our comparative evaluation focuses on the
canonical problem of navigating multiple robots to their goals
while avoiding collisions with environmental clutter and with each
other. The task can be specified as an instance of \eqref{eq:jocp}, for all
$t \in [ 0, T ]$,
\begin{equation}
\label{eq:sim-kino}
\begin{aligned}
  \min_{u^{1:N}(\cdot)} \
  & \textstyle\sum_{i} \sum_{k=0}^{T-1}
    \varrho_{i}\bigl( x^{i}( k + 1 ) \bigr)
    \\[0.2em]
  \text{s.t.} \
  & \dot{x}^{i} = f^{i}( x^{i}, u^{i} ) ,
    \quad u^{i}(t) \in U_{i} , \\[0.2em]
  & x^{i}(t) \in \X^{i}_{\mathrm{free}} \cap E_{i} , \\[0.2em]
  & x^{i}( T ) \in R_{i} ,
    \qquad x^{i}( 0 ) = x^{i}_{\mathrm{init}} , \\[0.2em]
  & B_{i}\bigl( x^{i}(t) \bigr) \cap B_{j}\bigl( x^{j}(t) \bigr)
    = \emptyset ,
    \qquad \forall\, j \ne i ,
\end{aligned}
\end{equation}\endgroup
where we use the stage cost
\begin{equation}
\label{eq:settle}
  \varrho_{i}( x )
  := A\, \mathbf{1}\bigl[ d_{i}( x ) > \varepsilon_{g} \bigr]
  + \max\bigl\{ d_{i}( x ),\, \varepsilon_{g} \bigr\}
\end{equation}
charges $A = 10^{3}$ for every window that robot $i$ ends outside
the ball of radius $\varepsilon_{g}$ around its goal state
$x^{i}_{\mathrm{goal}}$, plus its distance $d_{i}( x )$ to the goal,
floored at $\varepsilon_{g}$. The free space
$\X^{i}_{\mathrm{free}}$ excludes the static obstacles and the
workspace boundary, and the state envelope $E_{i}$, the terminal rest
set $R_{i}$, and the body $B_{i}$ are listed in
\cref{tab:platforms} for each of the platforms we use.

\begin{remark}
\label{rem:settle}
The indicator in \eqref{eq:settle} makes the stage cost jump by $A$
at the boundary of the goal ball, violating the Lipschitz requirement
\eqref{eq:winagg}. Our results nonetheless transfer with a small patch
as follows. The proofs of \cref{thm:single} and \cref{prop:priced}
use the Lipschitz property only in \eqref{eq:ind-stage}, where a
trajectory that tracks the reference $z^{i}(\cdot)$ to within $\rho$
over a window incurs a stage cost at most $L_{\varrho}\, \rho$ above
that of the reference. If the jets of the reference and of the
tracking trajectory at the end of the window are $\xi$ and $\xi'$,
then $d( \xi', \xi ) \le \rho$ and \eqref{eq:settle} give
\[
  \varrho_{i}( \xi' )
  \le A\, \mathbf{1}\bigl[ d_{i}( \xi ) > \varepsilon_{g}
  - L_{d}\, \rho \bigr]
  + \max\bigl\{ d_{i}( \xi ),\, \varepsilon_{g} \bigr\}
  + L_{d}\, \rho ,
\]
with $L_{d}$ the Lipschitz constant of $d_{i}$, which takes the place of $L_{\varrho}$. The results thus
hold with the cost of the reference evaluated on the goal ball shrunk
by $L_{d}\, \rho$, which raises it by $A$ only for the windows
at the end of which the reference lies inside the ball within
$L_{d}\, \rho$ of its boundary. For the cost
\eqref{eq:settle}, the cost of the reference in
\cref{prop:jcs,prop:priced,thm:certificate} and
\cref{prop:anytime}(ii), and hence the $\rho$-robust optimum, is
accordingly evaluated on the shrunk goal ball.
\end{remark}

The command sets of \cref{tab:platforms} follow the construction
presented in~\cite[Sec.~6.2.2]{deshpande2026DFT}. For the unicycle, the single and double
integrators, and the X500, the commands hold a lattice quantization
of the flat input over the window, through the geometric tracking
controller for the X500, and are dispersive by the arguments
detailed in~\cite[App.~B]{deshpande2026DFT}, with a dispersion that vanishes as the
lattice and the windows are refined. For the triple integrator,
the commands  disperse the velocity reached at the end of
the window, over directions about the current heading that
successive windows extend.

\smallskip

\begin{table}[!htb]
\caption{Multi-robot Dynobench: total travel time (s) summed over agents, median. Quiesced: WarpWeaver with the world paused until $\varepsilon_{q}$-quiescence or for at most $5$\,s (\cref{sec:cutoff}), $\varepsilon_{q}$-quiescent at more than $84\%$ of the iterations of every run, median over five trials.}
\label{tab:dynobench}
\centering
\footnotesize
\fullwidthtab{9}{\begin{tabular}{|l|l|c|c|c|cccc|}
\hline
 & Instance & $N$ & \shortstack{Warp-\\Weaver} & \shortstack{Quie-\\sced} & db-CBS & db-ECBS & db-PIBT & \shortstack{db-LaCAM\\@60\,s} \\
\hline
\multirow{6}{*}{\rotatebox{90}{Unicycle}} & swap2 & $2$ & $\mathbf{11.0}$ & $11.0$ & $12.0$ & $14.4$ & $16.8$ & $15.6$ \\
  & swap4 & $4$ & $\mathbf{24.0}$ & $22.5$ & $27.6$ & $27.1$ & -- & $72.0$ \\
  & atgoal & $2$ & $\mathbf{6.5}$ & $7.0$ & $8.4$ & $8.3$ & $25.2$ & $16.8$ \\
  & circle6 & $6$ & $84.5$ & $83.5$ & $90.0$ & $\mathbf{82.9}$ & $97.2$ & $96.0$ \\
  & circle8 & $8$ & $\mathbf{113.0}$ & $113.0$ & $121.2$ & $117.8$ & $141.6$ & $136.8$ \\
  & circle10 & $10$ & $\mathbf{140.5}$ & $139.5$ & $150.0$ & $156.3$ & $182.4$ & $196.8$ \\
\hline
\multirow{5}{*}{\rotatebox{90}{Double Int.}} & swap2 & $2$ & $\mathbf{13.0}$ & $13.0$ & $14.0$ & $17.0$\,\cpart{3/5} & $15.0$\,\cpart{2/5} & $16.5$ \\
  & forest4 & $4$ & $\mathbf{26.0}$ & $26.0$ & $30.5$ & $34.5$ & -- & -- \\
  & passage6 & $6$ & $\mathbf{41.5}$ & $41.5$ & $48.5$ & $55.5$ & -- & $57.5$ \\
  & swap8 & $8$ & $\mathbf{44.0}$ & $44.0$ & $47.0$ & $54.0$ & -- & $59.5$ \\
  & forest10 & $10$ & $\mathbf{80.0}$ & $80.0$ & $93.0$ & $109.5$ & -- & -- \\
\hline
\end{tabular}}
\end{table}

\begin{table}[!htb]
\caption{Ten drones on the EGO-Swarm forest maps. $F$: total travel time (s) summed over agents, median (min--max) over successful trials.}
\label{tab:ego-fan}
\centering
\footnotesize
\fullwidthtab{4}{\begin{tabular}{|c|l|c|c|}
\hline
Map & Planner & Success & $F$ (min--max) \\
\hline
\multirow{3}{*}{Forest 1} & WarpWeaver (Triple Int.) & \cgood{5/5} & \cgood{91.5} (91.0--94.5) \\
  & WarpWeaver (X500) & \cgood{5/5} & \cmid{97.5} (97.5--101.0) \\
  & EGO-Swarm (Triple Int.) & \cmid{4/5} & \cbad{107.3} (105.0--110.0) \\
\hline
\multirow{3}{*}{Forest 2} & WarpWeaver (Triple Int.) & \cgood{5/5} & \cgood{92.5} (91.0--93.0) \\
  & WarpWeaver (X500) & \cgood{5/5} & \cmid{99.0} (98.0--102.0) \\
  & EGO-Swarm (Triple Int.) & \cmid{4/5} & \cbad{108.6} (107.5--113.1) \\
\hline
\multirow{3}{*}{Forest 3} & WarpWeaver (Triple Int.) & \cgood{5/5} & \cgood{90.5} (89.5--90.5) \\
  & WarpWeaver (X500) & \cgood{5/5} & \cmid{97.0} (96.0--103.5) \\
  & EGO-Swarm (Triple Int.) & \cmid{4/5} & \cbad{108.0} (104.3--115.0) \\
\hline
\end{tabular}}
\end{table}

\medskip\noindent\textbf{Comparative Evaluation and Results:} We evaluate on the two
platforms of the multi-robot Dynobench
suite~\cite{moldagalieva2024dbcbs}, the unicycle and the 3D double
integrator, against the centralized planners of the
discontinuity-bounded family, db-CBS~\cite{moldagalieva2024dbcbs},
db-ECBS~\cite{moldagalieva2025dbecbs}, db-PIBT, and
db-LaCAM~\cite{moldagalieva2025dblacam} (\cref{tab:dynobench}).
Moreover, the triple integrator and the X500 are used in our
evaluation against EGO-Swarm~\cite{zhou2021egoswarm}, with ten
drones on three of its forest maps (\cref{tab:ego-fan}). The tables
report the median over five trials per instance of the total travel
time summed over agents. On Dynobench, the baselines run with their released
primitives and a $60$\,s limit, reporting the best solution within
it for the anytime planners, and the same or relaxed goal tolerance
relative to WarpWeaver. Against EGO-Swarm, both planners receive the full forest
map and the same limits on velocity, acceleration, and jerk
($4$\,m/s, $8$\,m/s$^{2}$, and $25$\,m/s$^{3}$). EGO-Swarm is reported at its best performance after
considerable tuning relative to the default configuration. As collision
avoidance and dynamical constraints are enforced via penalties, EGO-Swarm
sparsely violates the limits on velocity and acceleration, and its three failures are runs in which a drone collides deeply with an obstacle.

On Dynobench (\cref{tab:dynobench}),
WarpWeaver completes every trial and attains the lowest total travel
time on ten of the eleven instances, improving on the best baseline
by $8\%$ at the median and by up to $22\%$, while committing in real
time against planners that solve the full problem offline. Against
EGO-Swarm (\cref{tab:ego-fan}), WarpWeaver succeeds in all fifteen
trials on both platforms, against twelve for EGO-Swarm, and reduces
the total travel time by $15$--$16\%$ with the triple integrator and
by $9$--$10\%$ with the X500.

Although we omit detailed ablations
against sub-optimal coordination, the atgoal task demonstrates the
general inadequacy of such mechanisms, as iterated best response over
the same forward trees fails in either priority order over the robots.

\medskip\noindent\textbf{Scaling Tests:} We evaluate the scaling performance of WarpWeaver as the number of
randomly initialized unicycles in an empty $20 \times 20$\,m
environment is scaled as $N = 10, 20, \dots, 50$, on the instances of
the scalability test of db-LaCAM~\cite[Sec.~5.2]{moldagalieva2025dblacam},
with five trials per $N$. To emulate faithfully without requiring our
host computer to scale linearly with the number of robots, we give
each robot a virtual clock that advances only by the measured wall time
of its own computation on its own share of the GPU and a CPU core. A
message is delivered once the clock of its receiver passes that of its
sender at sending, plus $1$\,ms, and the robot with the earliest clock
always runs next. This preserves the asynchrony of \cref{alg:dcg} while
the simulation runs about $N$ times slower than real time.

\begin{table}[!htb]
\caption{Scaling tests, unicycles in a $20 \times 20$\,m arena, five trials per $N$, WarpWeaver in real time (emulated). Total travel time summed over agents and per robot (s), median (min--max) for WarpWeaver and  the median best solution of db-LaCAM within $300$\,s.}
\label{tab:scale-rt}
\centering
\footnotesize
\fullwidthtab{5}{\begin{tabular}{|c|c|c|c|c|}
\hline
 & \multicolumn{2}{c|}{WarpWeaver} & \multicolumn{2}{c|}{db-LaCAM} \\
$N$ & Total & per robot & Total & per robot \\
\hline
$10$ & $167.5$ {\scriptsize($167.0$--$168.5$)} & $16.8$ & $217.2$ & $21.7$ \\
$20$ & $435.5$ {\scriptsize($432.5$--$438.0$)} & $21.8$ & $522.0$ & $26.1$ \\
$30$ & $649.5$ {\scriptsize($645.0$--$654.0$)} & $21.7$ & $884.4$ & $29.5$ \\
$40$ & $910.5$ {\scriptsize($904.0$--$913.0$)} & $22.8$ & $1400.4$ & $35.0$ \\
$50$ & $1126.5$ {\scriptsize($1107.5$--$1165.5$)} & $22.5$ & $1470.0$ & $29.4$ \\
\hline
\end{tabular}}
\end{table}

\Cref{tab:scale-rt} presents the total travel time of WarpWeaver
against  the best solution of db-LaCAM within its $300$\,s limit, run offline. WarpWeaver completes all $25$ trials without collision, with a
total travel time $17$--$35\%$ lower than that of db-LaCAM at every
$N$. Its travel time per robot rises from $16.8$\,s at $N = 10$ to
about $22$\,s at $N = 20$ and stays flat up to $N = 50$.

\subsection{Connectivity Maintenance}
\label{sec:sim-relay}

While collision avoidance only requires pairwise coordination, we
now evaluate on a task that requires coordination across the whole
team. In connectivity maintenance, the robots must reach their
individual targets while their communication graph stays connected,
with every link requiring both range and an unobstructed line of
sight through an environment mapped online. For four planar
single integrators, the task can be specified as an instance of
\eqref{eq:jocp}, for all $t \in [ 0, T ]$,
\begin{equation}
\label{eq:sim-relay}
\begin{aligned}
  \min_{u^{1:4}(\cdot)} \
  & \textstyle\sum_{i} \sum_{k=0}^{T-1}
    \bigl( 1 - \nu^{i} \bigr)\,
    \varrho_{i}\bigl( p^{i}( k + 1 ) \bigr)
    \\[0.2em]
  \text{s.t.} \
  & \dot{p}^{i} = u^{i} ,
    \quad \| u^{i}(t) \| \le 1.5 , \\[0.2em]
  & p^{i}(t) \in \X_{\mathrm{free}} ,
    \qquad p^{i}( 0 ) = p^{i}_{\mathrm{init}} , \\[0.2em]
  & \| p^{i}(t) - p^{j}(t) \| \ge 2 r ,
    \qquad \forall\, j \ne i , \\[0.2em]
  & \max_{S \in \mathcal{S}} \ \min_{\{ i, j \} \in S}
    \ell\bigl( p^{i}(t), p^{j}(t) \bigr) \ge 0 ,
\end{aligned}
\end{equation}
where we use the stage cost $\varrho_{i}$ of \eqref{eq:settle} with
$d_{i}( p )$ the shortest-path distance from $p$ to the target
$y_{i}$ on the map, and $\nu^{i} = 1$ once the executed trajectory
of robot $i$ has come within $\varepsilon_{g}$ of $y_{i}$, after which
the robot may leave its target to relay for the others. As $\nu^{i}$ is
fixed over each planning horizon, \cref{rem:settle} applies to this
cost as well. The last constraint is the connectivity
predicate of \cref{ex:relay} on the robots, with $\mathcal{S}$ the
spanning trees of the complete graph on the four robots and the
link margin
\[
  \ell( p, p' ) = \min\bigl\{ R_{c} - \| p - p' \|,\
  \operatorname{dist}\bigl( [ p, p' ],\, O \bigr) - \delta_{c} \bigr\} ,
\]
nonnegative when the two robots are within the range
$R_{c} = 30$\,m and their segment clears the blocked set $O$ by
$\delta_{c} = 0.2$\,m, and $r$ is the body
radius of \cref{tab:platforms}. The free space
$\X_{\mathrm{free}}$ is mapped online from the 2D lidar scans of the
robots, fused across the team, and $O$ is its complement, every cell
not yet observed free.

\begin{table}[!htb]
\caption{Connectivity maintenance against Realm, four robots. Time: arrival of the last robot (s), median (min--max); Conn.: fraction of time connected in the true world geometry, worst task; Path: summed path length (m), median.}
\label{tab:realm}
\centering
\footnotesize
\fullwidthtab{6}{\begin{tabular}{|l|l|c|c|c|c|}
\hline
World & Planner & Success & Time (s) & Conn. & Path \\
\hline
\multirow{2}{*}{Clutter} & WarpWeaver & \cgood{5/5} & \cgood{23.6} (19.6--32.9) & $1.00$ & $123$ \\
  & Realm & \cgood{5/5} & \cbad{58.3} (31.0--109.2) & $1.00$ & $193$ \\
\hline
\multirow{2}{*}{Garage} & WarpWeaver & \cgood{5/5} & \cgood{44.8} (31.7--54.9) & $1.00$ & $241$ \\
  & Realm & \cmid{3/5} & \cbad{61.9} (55.3--85.5) & $1.00$ & $226$ \\
\hline
\multirow{2}{*}{Forest} & WarpWeaver & \cgood{5/5} & \cgood{56.8} (42.0--73.4) & $0.98$ & $242$ \\
  & Realm & \cbad{0/5} & -- & $1.00$ & -- \\
\hline
\end{tabular}}
\end{table}

\medskip\noindent\textbf{Comparative Evaluation and Results:} We compare against Realm~\cite{bai2025realm}, a planner dedicated
to line-of-sight connectivity maintenance, at its released
configuration~\cite{bai2026polygonal} and in its simulation
environment, on three worlds with five tasks each (\cref{tab:realm}).
Both planners receive the same lidar scans and issue the same pose
commands, and WarpWeaver runs on the simulation clock, with its compute per
window capped at the wall-clock budget $\Delta - \tau -
\delta_{\mathrm{comm}}$, so that a slowed simulation grants it no
extra computation, while Realm runs its $30$\,Hz controller on the
CPU. A task fails
if the last robot has not arrived within $200$\,s of simulated
time. The table
reports the time at which the last robot reaches its target, the
fraction of time the graph is connected, and the total path length.

WarpWeaver completes all fifteen tasks against eight for Realm,
which solves none in the forest, where its leading and following robots
stall while connected, and brings the last robot to its
target $59\%$ sooner in the clutter world and $28\%$ sooner in the
garage at the median. Both planners keep the team connected
throughout, except for the $2\%$ loss reported in the forest task by
WarpWeaver due to mapping discrepancies between the true world
geometry and the simulated lidar artifacts. Robots that have reached their targets leave them to
relay for the others, a joint maneuver that the formulation
produces without a dedicated mechanism.

\subsection{The Cost of a Real-Time Cutoff}
\label{sec:cutoff}
\label{sec:sim-quiesc}

Since the real-time implementation of WarpWeaver commits the negotiated
solution as soon as the deadline is up, it is of some interest to
understand how far the executed performance of the real-time loop falls
from the anytime guarantee of \cref{prop:anytime}, which holds only at
quiescence. In our scaling tests, for instance, we observe that the
fraction of quiescent iterations in real-time emulations falls from
$70\%$ at $N = 10$ to none at $N = 50$ (\cref{tab:scale-tq}),
as each agent receives more columns, solves a larger restricted master,
and completes fewer pricing steps per window, from $11$ to $5$.

We therefore rerun the scaling tests in a deliberation mode, which
pauses the world at every iteration until the team is quiescent, or for
at most $5$\,s (to keep runtimes reasonable). Since exact quiescence is
reached within this cap at $98\%$ of the iterations at $N = 10$ but at
$2\%$ at $N = 50$, we also measure the time to $\varepsilon_{q}$-quiescence,
$\hat{J}_{i} \ge -\varepsilon_{q}$ for every agent $i$, with
$\varepsilon_{q} = 1$, under which the bound of \cref{thm:certificate}
loosens by $N \varepsilon_{q}$.

\begin{table}[!htb]
\caption{Quiescence in real time and with the world paused (cap $5$\,s, one trial per $N$). Quiesc.: iterations that end quiescent, in real time (median over the trials of \cref{tab:scale-rt}) and within the cap at $\varepsilon_{q} = 0$ and $\varepsilon_{q} = 1$; $t_q$: time to quiescence at $\varepsilon_{q} = 1$ (s), capped iterations counted as $> 5$; $F$: total travel time (s) in real time (median, \cref{tab:scale-rt}) and with up to $5$\,s of negotiation per iteration.}
\label{tab:scale-tq}
\centering
\footnotesize
\fullwidthtab{10}{\begin{tabular}{|c|ccc|ccc|ccc|}
\hline
 & \multicolumn{3}{c|}{Quiesc.} & \multicolumn{3}{c|}{$t_q$, $\varepsilon_{q} = 1$} & \multicolumn{3}{c|}{$F$} \\
$N$ & RT & $\varepsilon_{q}=0$ & $\varepsilon_{q}=1$ & p50 & p90 & p99 & RT & $\varepsilon_{q}=0$ & $\varepsilon_{q}=1$ \\
\hline
$10$ & $70\%$ & $98\%$ & $100\%$ & $0.11$ & $0.31$ & $1.26$ & $167.5$ & $167.5$ & $167.0$ \\
$20$ & $16\%$ & $58\%$ & $99\%$ & $1.32$ & $2.41$ & $>5$ & $435.5$ & $431.5$ & $435.0$ \\
$30$ & $6\%$ & $33\%$ & $93\%$ & $1.48$ & $4.63$ & $>5$ & $649.5$ & $643.0$ & $658.0$ \\
$40$ & $1\%$ & $7\%$ & $60\%$ & $3.56$ & $>5$ & $>5$ & $910.5$ & $899.5$ & $898.5$ \\
$50$ & $0\%$ & $2\%$ & $7\%$ & $>5$ & $>5$ & $>5$ & $1126.5$ & $1072.0$ & $1065.5$ \\
\hline
\end{tabular}}
\end{table}

Regardless, we find that the real-time loop commits to
solutions close to quiesced agreements. On the instances of
\cref{tab:dynobench}, its median total travel time matches the
quiesced one on seven of the eleven instances, and exceeds it by at
most $1.2\%$ on all but swap4 of the unicycle, by $1.5$\,s. In the
scaling tests, it exceeds that of the $\varepsilon_{q}$-quiesced agreements, at
$\varepsilon_{q} = 1$, by at most
$0.3\%$ up to $N = 30$ (\cref{tab:scale-tq}), and, at $N = 40$ and
$50$, where the paused negotiation rarely quiesces within the $5$\,s
cap, remains within $1.3\%$ and $6\%$ of the agreements after
$5$\,s of negotiation. The
real-time loop thus commits close to the solution that a longer
negotiation attains, and chasing quiescence only buys a slightly
improved solution alongside the certificate, as is typical of column
generation~\cite{lubbecke2005column}.

\section{Real-Robot Experiments}
\label{sec:hardware}

\begin{figure}[t]
\centering
\newcommand{\hrow}[4]{\makebox[\columnwidth]{%
  \includegraphics[height=#4]{#1}\hfill
  \includegraphics[height=#4]{#2}}\\[-1pt]
  {\footnotesize (#3)}\\[2pt]}
\hrow{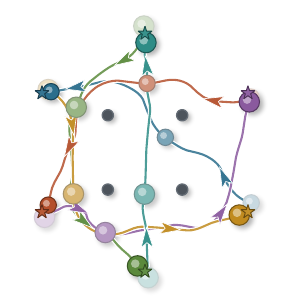}{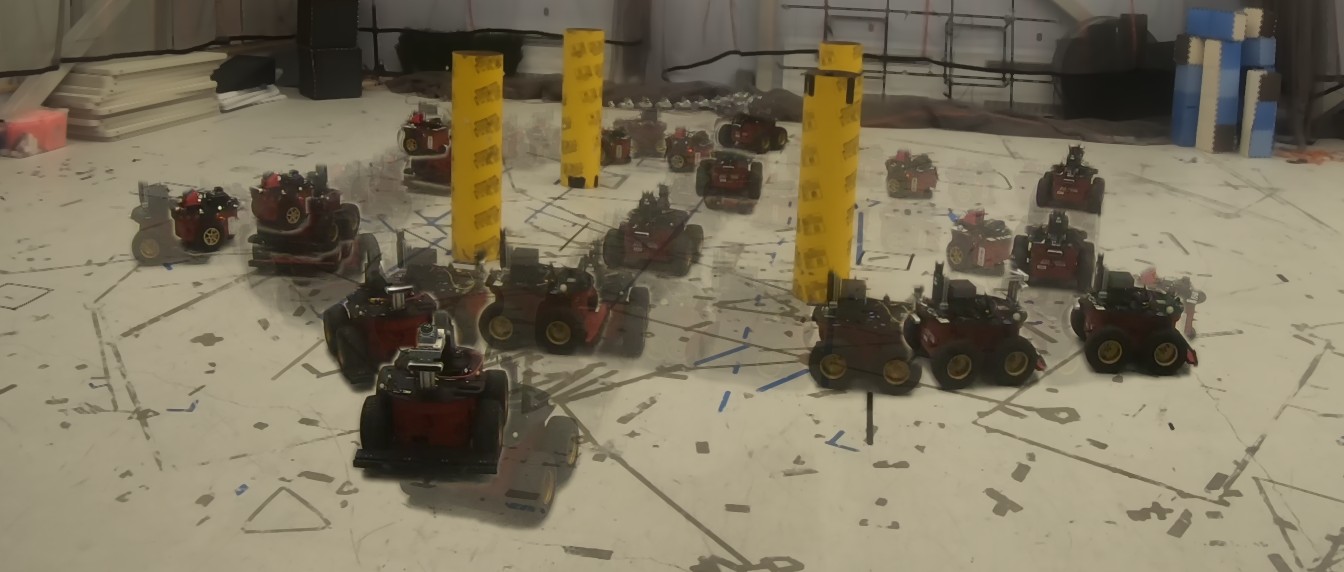}{a}{73pt}
\hrow{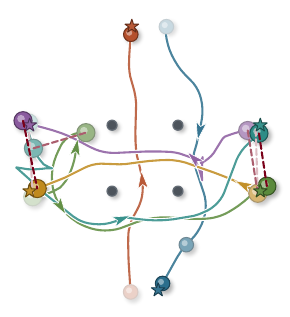}{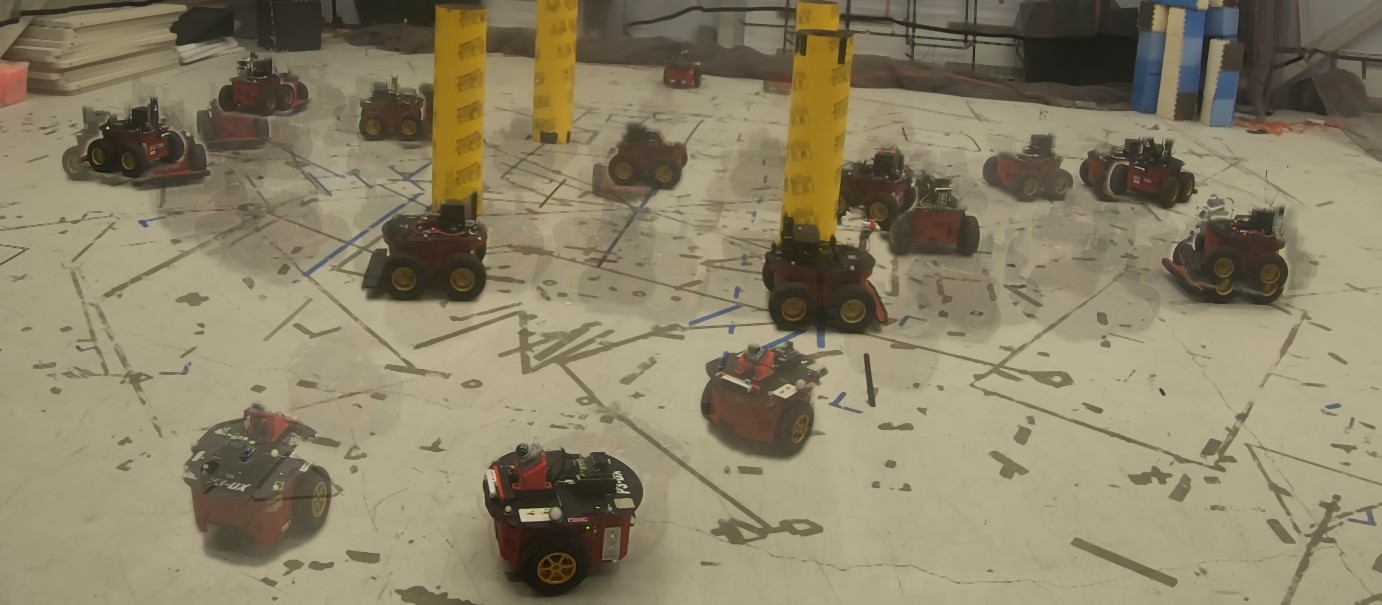}{b}{75pt}
\hrow{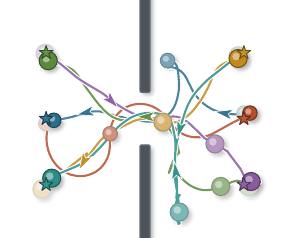}{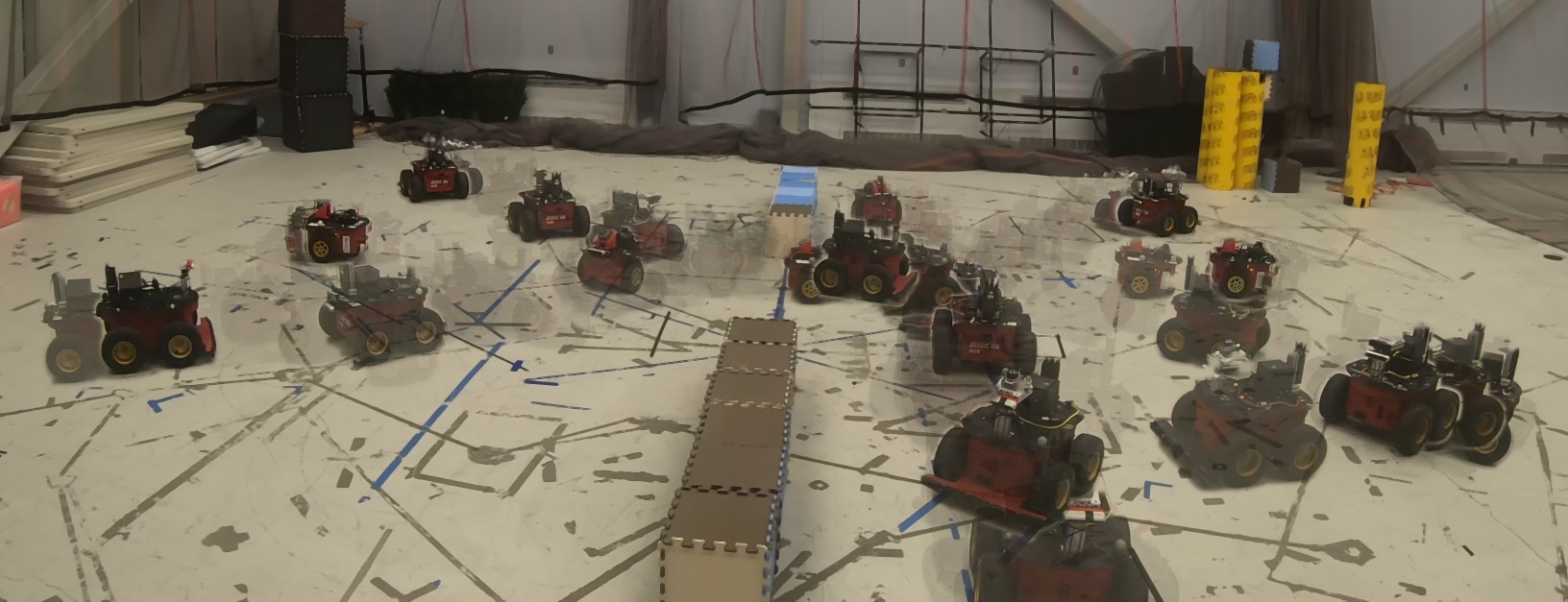}{c}{63pt}
\caption{Six robots, one take per task. Left: executed trajectories from motion capture, each robot in its own colour, with the robots at the start, midway, and at the last arrival, later instants more opaque. Right: the camera frames overlaid, later ones more opaque. (a) Swap around four pillars. (b) Tethered swap, with the tied pairs joined by dashed lines. (c) Doorway swap.}
\label{fig:hw}
\end{figure}

We deploy WarpWeaver on a team of six ground robots, each planning
onboard its own Jetson Orin and exchanging messages with the
others over Wi-Fi through ROS~2, with positions from a Vicon motion
capture system and a hand-entered obstacle map. The robots are of
two classes with unicycle dynamics
$\dot{p} = v e_{\theta}$, $\dot{\theta} = \omega$, two of them with
a clearance disc of radius $0.4$\,m and acceleration limit
$|\dot{v}| \le 1.5$\,m/s$^{2}$, and four others with $0.5$\,m and
$2.0$\,m/s$^{2}$, and $|\dot{\omega}| \le 2$\,rad/s$^{2}$ for both. We
evaluate three tasks (\cref{fig:hw}). In the swap, the six robots
exchange positions across a circle around four pillars of radius
$0.15$\,m. In the tethered swap, two pairs of robots in the same
scene are additionally tied to each other. In the doorway swap, two
groups of three exchange sides through a $1.6$\,m door in a wall.

Each task is an instance of \eqref{eq:sim-kino} on the unicycle
dynamics above, with $| v | \le v_{\max}$, $| \omega | \le
\omega_{\max}$, the disc bodies, and the stage cost
\eqref{eq:settle}. In the tethered swap, each tied pair $( i, j )$
further satisfies, for all $t \in [ 0, T ]$,
\[
  \min\bigl\{ L - \| p^{i}(t) - p^{j}(t) \|,\
  \operatorname{dist}\bigl( [ p^{i}(t), p^{j}(t) ],\, O \bigr)
  \bigr\} \ge 0 ,
\]
the link predicate of \cref{ex:relay} with $R_{c} = L = 2.4$\,m and
$\delta_{c} = 0$, which keeps each pair within range and in line of
sight. Each robot runs \cref{alg:dcg} over target commands $( v, \omega )$
on a $13 \times 13$ lattice of the input box, each reached from the
previous command by a ramp at the acceleration limits, so that the rollouts, the collision checks, and the tracked plans
respect these limits, with cells of diameter $S = s_{i}$
and $n_{s} = 5$ substeps per window, $\Delta = 1$\,s,
$\Delta_{s} = 1$\,s, $\tau = 0.2$\,s, $\delta_{\mathrm{comm}} = 0.2$\,s, and
$\varepsilon_{g} = 0.25$\,m. The swap uses $T = 10$, $s_{i} = 0.25$,
$v_{\max} = 0.5$\,m/s, and $\omega_{\max} = 0.5$\,rad/s, the tethered
swap $T = 8$, $s_{i} = 0.35$, and $v_{\max} = 0.8$\,m/s, and the
doorway swap $T = 6$, $s_{i} = 0.25$, $v_{\max} = 0.8$\,m/s for the four and
$1.0$\,m/s for the two, and $\omega_{\max} = 1.0$\,rad/s. Over $71{,}800$ round
trips of $2$\,KB messages among the six onboard computers during three
runs, the round-trip time had a median of $13$\,ms and a $99$th
percentile of $106$\,ms, and exceeded $0.2$\,s for $0.2\%$ of them,
with a maximum of $635$\,ms. Though some exceeded $\delta_{\mathrm{comm}}$, this
led to no observed inconsistency in the commits. The robots committed the same agreement at
$311$ of the $312$ picks, the exception following the restart of a
robot after a battery swap.

\begin{table}[!htb]
\caption{Real-robot experiments, six robots, three consecutive takes per task. Time: arrival of the last robot (s) and total travel time summed over robots (s), median (min--max); Sep.: smallest distance between two robots minus both clearance radii (m); Tether: largest distance within a tied pair (m), against $L = 2.4$\,m.}
\label{tab:hw}
\centering
\footnotesize
\fullwidthtab{6}{\begin{tabular}{|l|c|c|c|c|c|}
\hline
Task & Success & Last arrival & Total & Sep. & Tether \\
\hline
Swap & $3/3$ & $20.2$ ($20.0$--$20.8$) & $108.1$ & $0.28$ & -- \\
Tethered swap & $3/3$ & $30.4$ ($24.5$--$31.9$) & $125.4$ & $0.28$ & $2.37$ \\
Doorway swap & $3/3$ & $33.6$ ($25.3$--$35.6$) & $162.1$ & $0.12$ & -- \\
\hline
\end{tabular}}
\end{table}

Every robot reached its goal in every take, without contact between
robots or with the obstacles (\cref{tab:hw}). The tied pairs stayed
within range and in line of sight throughout, reaching at most
$2.37$\,m of the $2.4$\,m tether, and the executed trajectories
tracked the committed plans to within $0.03$--$0.07$\,m on average.

\section{Conclusion}
\label{sec:conclusion}

In this paper, we formulated a broad class of multi-robot
coordination problems as joint-space optimal control subject to
constraints over groups of robots, and extended a recent single-robot
framework of dispersive forward tree search to approach its joint
optimum. We showed that suitable constraints on each robot can effectively
replace the inter-robot coupling, so that dispersive forward
trees grown locally by the robots jointly contain a near-optimal
agreement. This structure was pursued by decentralized column
generation, in which the robots search their trees under common
prices and negotiate an agreement certified near-optimal at
quiescence. This development culminated in WarpWeaver, a real-time
protocol that runs the negotiation asynchronously over a receding
horizon onboard each robot. In simulation, WarpWeaver was shown to
match or outperform state-of-the-art centralized offline,
decentralized real-time, and task-specific planners. Its total travel times were
up to $22\%$ lower than those of
centralized offline planners on Dynobench, and up to $16\%$ lower
than those of a decentralized real-time planner, and it nearly
doubled the success rate of a task-specific planner on graph
connectivity. Notably, we find
that WarpWeaver demonstrates effective closure to quiesced
agreements even when real-time deadlines cut the negotiation
short, implying that decentralized column generation is well
suited for real-time multi-robot coordination. In real-robot
experiments, we used a team of six unicycle robots to demonstrate
its real-time implementability for kinodynamic planning around
obstacles and through a doorway, and additionally with pairs of
robots tethered to each other.

A limitation of our approach is that a coordination task is only
as tractable as its integer linear restriction. Constraints that
tie many robots tightly together, such as the maintenance of a rigid
formation, admit few joint solutions among the individual trees and
yield programs whose relaxations are weak, where differential
methods that exploit the local geometry of the constraint may be
better suited. Incorporating such differential information into
dispersive search remains a direction for future research.

\appendices
\section{Proofs}
\label[appendix]{app:proofs}

\subsection{Proof of Proposition~\ref{prop:jcs}}
\label{app:jcs}

For each $i$, let $g^{z, i}$ be the individual predicate of the
form \eqref{eq:winsep} with the windowwise predicates
\begin{equation}
\label{eq:corridor}
  g^{z, i}_{k}( w )
  := L_{\preds} \Bigl( \rho
  - d\bigl( w,\; \jet z^{i}|_{[k,\, k+1]} \bigr) \Bigr) ,
\end{equation}
which keeps agent $i$ within $\rho$ of its reference, and let
$\preds'_{i} := \preds_{i} \cup \{ g^{z, i} \}$. The reference
satisfies $g^{z, i}$ with the margin
$g^{z, i}( \jet z^{i} ) = L_{\preds}\, \rho
\ge L_{\preds}\, \rho_{i}$, and the predicates $\preds_{i}$
with the margin \eqref{eq:refmargin}.
\Cref{thm:single}, applied to the problem \eqref{eq:socp} of
agent $i$ under $\preds'_{i}$, thus gives a sequence
$\mathbf{a}^{i}$ in $\cl_{i}$ that satisfies every
$g \in \preds'_{i}$, with
$J_{i}( \roll_{i}( \mathbf{a}^{i} ) )
\le J_{i}( \jet z^{i} ) + L_{J}\, \rho_{i}$. Since
$\alpha_{\ell i} \ge 0$ and $\rho_{i} \le \rho$,
\eqref{eq:aggregator} gives
\[
\begin{aligned}
  \Phi\bigl( \roll_{1}( \mathbf{a}^{1} ), \dots \bigr)
  &\le \max_{\ell} \sum_{i} \alpha_{\ell i}
  \bigl( J_{i}( \jet z^{i} ) + L_{J}\, \rho \bigr) \\
  &\le \Phi\bigl( \jet z^{1}, \dots \bigr)
  + L_{\Phi}\, \rho .
\end{aligned}
\]
It remains to show that $( \mathbf{a}^{1}, \dots, \mathbf{a}^{N} )$
satisfies the joint predicates. As it satisfies every $g^{z, i}$,
$d( \roll_{i}( \mathbf{a}^{i} )( t ), \jet z^{i}( t ) )
\le \rho$ for every $i$ and $t$. For $g \in \preds_{e}$ of the
form \eqref{eq:instant}, the Lipschitz property of $h_{t}$ and
the margin \eqref{eq:refmargin} then give, for every $t$,
\[
  h_{t}\bigl( ( \roll_{j}( \mathbf{a}^{j} )( t ) )_{j \in e}
  \bigr)
  \ge h_{t}\bigl( ( \jet z^{j}( t ) )_{j \in e} \bigr)
  - L_{\preds}\, \rho
  \ge 0 .
\]
Hence $( \mathbf{a}^{1}, \dots, \mathbf{a}^{N} )
\in \prod_{i} \cl_{i}$ is an agreement.

\subsection{Proof of Proposition~\ref{prop:ilp}}
\label{app:ilp}

We recall that the occupancy rows of \eqref{eq:ilp} forbid the
members of a group $e$ to occupy, during a substep $m$, the cells
of a tuple on which a joint predicate $g \in \preds_{e}$ fails,
\begin{equation}
\tag{\ref{eq:forbidden}}
  F_{m}( g )
  = \Bigl\{ ( \kappa_{j} )_{j \in e} :
  \inf_{t \in I_{m}}
  \inf_{\xi_{j} \in \kappa_{j},\, j \in e}
  h_{t}\bigl( ( \xi_{j} )_{j \in e} \bigr) < 0
  \Bigr\} ,
\end{equation}
and that jet curves $\zeta^{1}, \dots, \zeta^{N}$ obey them, that
is, occupy no forbidden tuple, when
\begin{equation}
\tag{\ref{eq:noforbid}}
  \sum_{j \in e}
  \mathbf{1}\bigl[ \kappa_{j} \in \occ_{j}( \zeta^{j}, m ) \bigr]
  \le | e | - 1
  \quad \forall ( g, m, \kappa ) \in \res .
\end{equation}
We first show that these rows discretize the joint predicates,
soundly and completely up to the discretization margin, in the
following sense.

\begin{lemma}
\label{lem:disc}
\noindent (a) If jet curves $\zeta^{1}, \dots, \zeta^{N}$
of the agents occupy no forbidden tuple
\eqref{eq:noforbid}, then
$g( ( \zeta^{j} )_{j \in e} ) \ge 0$ for all
$g \in \preds_{e}$ and all $e$.

\noindent (b) Let $z^{1}(\cdot), \dots, z^{N}(\cdot)$ be
$M'$-smooth with
$g( ( \jet z^{j} )_{j \in e} )
\ge L_{\preds}\, ( \eta + S + M' h )$ for all
$g \in \preds_{e}$ and all $e$, and let jet curves
$\zeta^{1}, \dots, \zeta^{N}$ satisfy
$d( \zeta^{i}( t ), \jet z^{i}( t ) ) \le \eta$ for all
$i$ and $t$. Then they occupy no forbidden tuple
\eqref{eq:noforbid}.
\end{lemma}

\begin{proof}
(a) Suppose $h_{t}( ( \zeta^{j}( t ) )_{j \in e} ) < 0$
for some $g \in \preds_{e}$ and $t \in I_{m}$. Let
$\kappa_{j}$ be the cell containing $\zeta^{j}( t )$. Then
$\kappa_{j} \in \occ_{j}( \zeta^{j}, m )$ for every
$j \in e$, and $( \kappa_{j} )_{j \in e} \in F_{m}( g )$
by \eqref{eq:forbidden}, contradicting \eqref{eq:noforbid}
at the row $( g, m, ( \kappa_{j} )_{j \in e} ) \in \res$.

(b) Suppose \eqref{eq:noforbid} fails at some row
$r = ( g, m, \kappa ) \in \res$, so that
$\zeta^{j}( t_{j} ) \in \kappa_{j}$ for some
$t_{j} \in I_{m}$, $j \in e$, and, as
$\kappa \in F_{m}( g )$,
$h_{t}( ( \xi_{j} )_{j \in e} ) < 0$ for some
$t \in I_{m}$ and $\xi_{j} \in \kappa_{j}$.
By the triangle inequality and the $M'$-smoothness of
$z^{j}$,
\[
\begin{aligned}
  d\bigl( \xi_{j}, \jet z^{j}( t ) \bigr)
  &\le S + d\bigl( \zeta^{j}( t_{j} ), \jet z^{j}( t_{j} )
  \bigr) \\
  &\qquad + d\bigl( \jet z^{j}( t_{j} ), \jet z^{j}( t )
  \bigr) \\
  &\le S + \eta + M' h ,
\end{aligned}
\]
so that, by the Lipschitz property of $h_{t}$ in the
maximum metric,
$h_{t}( ( \xi_{j} )_{j \in e} )
\ge h_{t}( ( \jet z^{j}( t ) )_{j \in e} )
- L_{\preds}\, ( \eta + S + M' h ) \ge 0$, a
contradiction.
\end{proof}

\noindent We now prove the two parts of \cref{prop:ilp}.

(i) Let $( \vartheta, \lambda )$ be a feasible solution of
\eqref{eq:ilp}. As $\lambda_{i, c} \in \{ 0, 1 \}$ and
$\sum_{c} \lambda_{i, c} = 1$, each agent $i$ selects exactly one
sequence $c_{i}$, with $\lambda_{i, c_{i}} = 1$ and
$\lambda_{i, c} = 0$ for $c \ne c_{i}$, and the constraints of
\eqref{eq:ilp} reduce to constraints on
$( c_{1}, \dots, c_{N} )$. Those on the individual predicates
give $g( \roll_{i}( c_{i} ) ) \ge 0$ for every
$g \in \preds_{i}$ and $i$. The occupancy rows give
$\sum_{j \in e} \mathbf{1}[ \kappa_{j} \in \occ_{j}( c_{j}, m ) ]
\le | e | - 1$ for every $( g, m, \kappa ) \in \res$, so that
the rollouts $\zeta^{j} := \roll_{j}( c_{j} )$ occupy no
forbidden tuple \eqref{eq:noforbid}, and satisfy the joint
predicates by \cref{lem:disc}(a). Thus
$( c_{1}, \dots, c_{N} )$ is an agreement. Finally, the
constraints on $\vartheta$ reduce to
$\vartheta \ge \sum_{i} \alpha_{\ell i} J_{i}( c_{i} )$,
$\ell = 1, \dots, L$, so that, by \eqref{eq:aggregator},
\[
  \vartheta \ge \max_{\ell} \sum_{i} \alpha_{\ell i}
  J_{i}\bigl( \roll_{i}( c_{i} ) \bigr)
  = \Phi\bigl( \roll_{1}( c_{1} ), \dots,
  \roll_{N}( c_{N} ) \bigr) .
\]

(ii) Let $\lambda_{i, c} := \mathbf{1}[ c = \mathbf{a}^{i} ]$ and
$\vartheta := \Phi( \roll_{1}( \mathbf{a}^{1} ), \dots,
\roll_{N}( \mathbf{a}^{N} ) )$. Then $\lambda$ is binary with
$\sum_{c} \lambda_{i, c} = 1$, and the constraints on the
individual predicates hold, as each $\mathbf{a}^{i}$ satisfies
every $g \in \preds_{i}$. The rollouts
$\zeta^{i} := \roll_{i}( \mathbf{a}^{i} )$ are $2 M$-smooth and
satisfy the joint predicates with the margin
$\gamma_{S} = L_{\preds}\, ( S + 2 M h )$, so that
\cref{lem:disc}(b), with $\jet z^{i} := \zeta^{i}$, $\eta = 0$,
and $M' = 2 M$, shows that they occupy no forbidden tuple. As
the occupancy rows reduce to \eqref{eq:noforbid} for these
rollouts, they hold. Finally, the constraints on $\vartheta$
reduce to
$\sum_{i} \alpha_{\ell i} J_{i}( \roll_{i}( \mathbf{a}^{i} ) )
\le \vartheta$, which holds by \eqref{eq:aggregator}. Hence
$( \vartheta, \lambda )$ is feasible for \eqref{eq:ilp}, whose
optimal value is at most $\vartheta$.

\subsection{Proof of Proposition~\ref{prop:lagrangian}}
\label{app:lagrangian}

(i) We recall from \eqref{eq:reduced-curve}, \eqref{eq:rent-window},
and \eqref{eq:rent} that
\[
\begin{gathered}
  \bar{J}_{i}( \zeta )
  = w_{i}\, J_{i}( \zeta ) + \sum_{k} r_{i, k}( \zeta ) - \pi_{i} ,
  \qquad
  w_{i} = \sum_{\ell} \beta_{\ell}\, \alpha_{\ell i} ,
  \\[0.2em]
  r_{i, k}( \zeta )
  = \sum_{m \,:\, I_{m} \subseteq [k,\, k+1]}\
  \sum_{\kappa \in \occ_{i}( \zeta, m )} \mu_{i, \kappa, m} ,
  \\[0.2em]
  \mu_{i, \kappa, m}
  = \sum_{\substack{r = ( g, m, \kappa' ) \in \res \\
  i \in e,\ \kappa'_{i} = \kappa}} \mu_{r} ,
\end{gathered}
\]
where $e$ is the group of the predicate $g \in \preds_{e}$ of the
row $r$. We first derive an upper bound on the total rent
$\sum_{i} \sum_{k} r_{i, k}( \zeta^{i} )$. As each substep lies
in exactly one window, we have, reordering the summation,
\begin{equation}
\label{eq:rowsum}
\begin{aligned}
  \sum_{i} \sum_{k} r_{i, k}( \zeta^{i} )
  &= \sum_{i} \sum_{m}\
  \sum_{\kappa \in \occ_{i}( \zeta^{i}, m )}
  \mu_{i, \kappa, m} \\
  &= \sum_{r = ( g, m, \kappa ) \in \res} \mu_{r}
  \sum_{i \in e}
  \mathbf{1}\bigl[ \kappa_{i} \in \occ_{i}( \zeta^{i}, m )
  \bigr] .
\end{aligned}
\end{equation}
Since the curves $\zeta^{1}, \dots, \zeta^{N}$ occupy no
forbidden tuple, the inner sum following $\mu_{r}$ is at most
$| e | - 1 = b_{r}$ by \eqref{eq:noforbid}, and given
$\mu \ge 0$,
\[
  \sum_{i} \sum_{k} r_{i, k}( \zeta^{i} )
  \;\le\; b^{\top} \mu .
\]
Next, as $\beta \ge 0$ and $\mathbf{1}^{\top} \beta = 1$, the
weighted costs form a convex combination of the terms
maximized in \eqref{eq:aggregator},
\[
\begin{aligned}
  \sum_{i} w_{i}\, J_{i}( \zeta^{i} )
  &= \sum_{\ell} \beta_{\ell} \sum_{i} \alpha_{\ell i}\,
  J_{i}( \zeta^{i} ) \\
  &\le \max_{\ell} \sum_{i} \alpha_{\ell i}\,
  J_{i}( \zeta^{i} )
  = \Phi( \zeta^{1}, \dots, \zeta^{N} ) .
\end{aligned}
\]
Summing \eqref{eq:reduced-curve} over $i$ and applying the two
upper bounds above,
\[
\begin{aligned}
  \sum_{i} \bar{J}_{i}( \zeta^{i} )
  &= \sum_{i} w_{i}\, J_{i}( \zeta^{i} )
  + \sum_{i} \sum_{k} r_{i, k}( \zeta^{i} )
  - \mathbf{1}^{\top} \pi \\
  &\le \Phi( \zeta^{1}, \dots, \zeta^{N} )
  + b^{\top} \mu - \mathbf{1}^{\top} \pi ,
\end{aligned}
\]
which rearranges to the claim.

(ii) The second inequality,
$\mathrm{LP}( \mathcal{C} ) \le \mathrm{LP}( Q )$, follows as
$\mathrm{LP}( Q )$ is a restriction of $\mathrm{LP}( \mathcal{C} )$
with the additional constraints $\lambda_{( i, c )} = 0$,
$( i, c ) \notin Q$. For the first, let
$( \vartheta, \lambda, s )$ be a feasible solution of
$\mathrm{LP}( \mathcal{C} )$. By \eqref{eq:reduced} and the
constraints $\vartheta \mathbf{1} \ge \mathbf{J} \lambda$,
$\mathbf{E} \lambda = \mathbf{1}$, and $\mathbf{O} \lambda \le b {} + s$
of \eqref{eq:master}, with $\beta \ge 0$,
$\mathbf{1}^{\top} \beta = 1$, $0 \le \mu \le P\, \mathbf{1}$, and $s \ge 0$,
\[
\begin{aligned}
  \sum_{( i, c ) \in \mathcal{C}} \lambda_{( i, c )}\,
  \bar{J}_{i}( c )
  &= \beta^{\top} \mathbf{J} \lambda
  - \pi^{\top} \mathbf{E} \lambda
  + \mu^{\top} \mathbf{O} \lambda \\
  &\le \vartheta {} + P\, \mathbf{1}^{\top} s - \mathbf{1}^{\top} \pi + b^{\top} \mu ,
\end{aligned}
\]
Moreover, by \eqref{eq:blocks}, the constraints
$\mathbf{G} \lambda \ge 0$ read
$g( \roll_{i}( c ) )\, \lambda_{( i, c )} \ge 0$ and force
$\lambda_{( i, c )} = 0$ for every sequence $c$ of agent $i$
that violates one of its individual predicates. With
$\mathbf{E} \lambda = \mathbf{1}$, the weights
$\lambda_{( i, c )}$ of each agent thus form a probability
distribution over its locally feasible sequences, whence, by \eqref{eq:pricing},
\[
  \sum_{( i, c ) \in \mathcal{C}} \lambda_{( i, c )}\,
  \bar{J}_{i}( c )
  \;\ge\; \sum_{i} \bar{J}_{i}^{*}
  = \sum_{i} \bar{J}_{i}( c_{i} ) .
\]
The two bounds give
$\mathbf{1}^{\top} \pi - b^{\top} \mu
+ \sum_{i} \bar{J}_{i}( c_{i} ) \le \vartheta {} + P\, \mathbf{1}^{\top} s$ for every feasible
solution of $\mathrm{LP}( \mathcal{C} )$, and hence the first
inequality.

(iii) If the dual solution is optimal, then
$\mathbf{1}^{\top} \pi - b^{\top} \mu = \mathrm{DP}( Q )
= \mathrm{LP}( Q )$ by strong duality. With
$\bar{J}_{i}( c_{i} ) = \bar{J}_{i}^{*} = 0$ for all $i$, the
inequalities of (ii) then read
$\mathrm{LP}( Q ) \le \mathrm{LP}( \mathcal{C} )
\le \mathrm{LP}( Q )$.

\subsection{Proof of Proposition~\ref{prop:priced}}
\label{app:priced}

Recall that, in \cref{thm:single}, the cost
$J_{i}( x^{i}(\cdot) )$ of a trajectory of agent $i$ is the sum of
the $L_{\varrho}$-Lipschitz stage costs of its window trajectories,
$J_{i}( x^{i}(\cdot) ) = \sum_{k} \varrho^{i}_{k}( x^{i}|_{[k,\, k+1]} )$
by \eqref{eq:winagg}. In \cref{prop:priced}, we deal with a
single-agent problem of a similar kind, with the priced running
cost \eqref{eq:priced-running},
\[
  \Sigma^{\mu}_{i}( \mathbf{a} )
  = w_{i}\, \Sigma_{i}( \mathbf{a} )
  + \sum_{j=1}^{k} r_{i, j-1}\bigl( \roll_{i}( \mathbf{a} ) \bigr)
  - \pi_{i} ,
\]
violating the structure imposed by \cref{thm:single} due to the
rent $r_{i, k}$ charged for each window, which is not Lipschitz
in the window trajectory. The proof of \cref{prop:priced} is
therefore an extension of the proof of \cref{thm:single} to
handle this deviation in cost structure. We first briefly recall
the proof of \cref{thm:single} detailed in
\cite{deshpande2026DFT}.

Let
$\varrho^{*}_{k} := \varrho^{i}_{k}( \jet z^{i}|_{[k,\, k+1]} )$
denote the stage costs of the reference $z^{i}(\cdot)$ in the
proposition statement. The proof shows, by induction over the depth $k$, that the tree
retains a sequence $\mathbf{a}_{k} \in \A_{i}^{k}$ with
$d( \roll_{i}( \mathbf{a}_{k} )( k ), \jet z^{i}( k ) )
\le \delta_{\A_{i}} + s_{i}$, whose running cost satisfies
\begin{equation}
\label{eq:ind-cost}
  \Sigma_{i}( \mathbf{a}_{k} )
  \;\le\; \sum_{j < k} \varrho^{*}_{j}
  + k\, L_{\varrho}\, \rho_{i} .
\end{equation}
In the inductive step from the depth $k$ to $k + 1$, the erosion of the
reference control and the
dispersion of $\A_{i}$ furnish a command $a \in \A_{i}$ whose
window tracks the reference. Specifically, from the end state of
$\mathbf{a}_{k}$, within $\delta_{\A_{i}} + s_{i}$ of the reference,
the margin correction lemma~\cite[Lemma~5.9]{deshpande2026DFT} gives an
admissible corrector control that rejoins the reference at $k + 1$ and
stays within $C_{B,i} ( \delta_{\A_{i}} + s_{i} )$ of it over the
window, and the dispersion of $\A_{i}$ gives a command within
$\delta_{\A_{i}}$ of this corrector, so that
\begin{equation}
\label{eq:ind-track}
  d\bigl( \roll_{i}( ( \mathbf{a}_{k}, a ) )|_{[k,\, k+1]} ,\,
  \jet z^{i}|_{[k,\, k+1]} \bigr)
  \;\le\; \rho_{i} ,
\end{equation}
with
$d( \roll_{i}( ( \mathbf{a}_{k}, a ) )( k + 1 ),
\jet z^{i}( k + 1 ) ) \le \delta_{\A_{i}}$. By the margin and the smoothness of the
reference, the window of the extension $( \mathbf{a}_{k}, a )$
satisfies $g_{k} \ge 0$ for every $g \in \preds_{i}$ and is
$2 M_{i}$-smooth, as \cref{def:mincost}(iii) requires. By
\eqref{eq:ind-track} and the $L_{\varrho}$-Lipschitz property
\eqref{eq:winagg} of the stage costs, its stage cost satisfies
\begin{equation}
\label{eq:ind-stage}
  \varrho^{i}_{k}\bigl( \roll_{i}( ( \mathbf{a}_{k}, a ) )|_{[k,\, k+1]}
  \bigr)
  \;\le\; \varrho^{*}_{k} + L_{\varrho}\, \rho_{i} .
\end{equation}
The pruning procedure of \cref{def:mincost}(iii) then retains, in the cell
of $\roll_{i}( ( \mathbf{a}_{k}, a ) )( k + 1 )$, a sequence
$\mathbf{a}_{k+1}$ of no greater running cost, which thus
satisfies \eqref{eq:ind-cost} at the depth $k + 1$ and
$d( \roll_{i}( \mathbf{a}_{k+1} )( k + 1 ), \jet z^{i}( k + 1 ) )
\le \delta_{\A_{i}} + s_{i}$.

Controlling the stage cost of the extension relative to that of
the reference, as in \eqref{eq:ind-stage}, relies on the Lipschitz
property, which no longer holds for the rent $r_{i, k}$. We
instead contain the rent of the extension over the window
$[k, k+1]$ through the tracking \eqref{eq:ind-track}.

\noindent For every substep
$I_{m} \subseteq [k, k+1]$, \eqref{eq:ind-track} and
$\rho_{i} \le \rho$ give
$d( \roll_{i}( ( \mathbf{a}_{k}, a ) )( t ), \jet z^{i}( t ) )
\le \rho$ at every $t \in I_{m}$, so that every cell that the
extension occupies at the substep $m$ lies in the $\rho$-inflated
occupancy $\occ^{\rho}_{i}( \jet z^{i}, m )$ of the reference. As
$\mu \ge 0$,
\begin{equation}
\label{eq:ind-rent}
  r_{i, k}\bigl( \roll_{i}( ( \mathbf{a}_{k}, a ) ) \bigr)
  \;\le\; \sum_{m \,:\, I_{m} \subseteq [k,\, k+1]}\
  \sum_{\kappa \in \occ^{\rho}_{i}( \jet z^{i}, m )}
  \mu_{i, \kappa, m} .
\end{equation}

The proof of \cref{prop:priced} now follows as described below.
The arguments surrounding the tracking \eqref{eq:ind-track}, the
predicates $g_{k} \ge 0$ for $g \in \preds_{i}$, and the
$2 M_{i}$-smoothness of the windows carry over from the proof of
\cref{thm:single}. It remains to establish the extended
counterpart of \eqref{eq:ind-cost},
\begin{equation}
\label{eq:ind-priced}
\begin{aligned}
  \Sigma^{\mu}_{i}( \mathbf{a}_{k} )
  \;\le\;\; & w_{i} \Bigl( \sum_{j < k} \varrho^{*}_{j}
  + k\, L_{\varrho}\, \rho \Bigr) \\
  & + \sum_{m \,:\, I_{m} \subseteq [0,\, k]}\
  \sum_{\kappa \in \occ^{\rho}_{i}( \jet z^{i}, m )}
  \mu_{i, \kappa, m} - \pi_{i} .
\end{aligned}
\end{equation}
At $k = 0$, both sides of \eqref{eq:ind-priced} equal $-\pi_{i}$.
Now suppose that \eqref{eq:ind-priced} holds at the depth $k$. As
the extension $( \mathbf{a}_{k}, a )$ shares its first $k$ windows
with $\mathbf{a}_{k}$, \eqref{eq:priced-running} gives
\begin{equation}
\label{eq:ind-split}
\begin{aligned}
  \Sigma^{\mu}_{i}( ( \mathbf{a}_{k}, a ) )
  &= \Sigma^{\mu}_{i}( \mathbf{a}_{k} ) \\
  &\quad + w_{i}\, \varrho^{i}_{k}\bigl(
  \roll_{i}( ( \mathbf{a}_{k}, a ) )|_{[k,\, k+1]} \bigr) \\
  &\quad + r_{i, k}\bigl( \roll_{i}( ( \mathbf{a}_{k}, a ) ) \bigr) .
\end{aligned}
\end{equation}
As $w_{i} \ge 0$ and $\rho_{i} \le \rho$, \eqref{eq:ind-stage}
gives
\[
  w_{i}\, \varrho^{i}_{k}\bigl(
  \roll_{i}( ( \mathbf{a}_{k}, a ) )|_{[k,\, k+1]} \bigr)
  \;\le\; w_{i} \bigl( \varrho^{*}_{k} + L_{\varrho}\, \rho \bigr) .
\]
Inserting this inequality, \eqref{eq:ind-rent}, and
\eqref{eq:ind-priced} at the depth $k$ into \eqref{eq:ind-split}
yields \eqref{eq:ind-priced} at the depth $k + 1$ for the
extension. The pruning procedure
retains a sequence $\mathbf{a}_{k+1}$ with
$\Sigma^{\mu}_{i}( \mathbf{a}_{k+1} )
\le \Sigma^{\mu}_{i}( ( \mathbf{a}_{k}, a ) )$, which therefore
satisfies \eqref{eq:ind-priced} at the depth $k + 1$.

At the depth $T$, we have
$\sum_{j < T} \varrho^{*}_{j} = J_{i}( \jet z^{i} )$ by
\eqref{eq:winagg}, and $T L_{\varrho} = L_{J}$. For the retained
sequence $c_{i} := \mathbf{a}_{T}$, \eqref{eq:ind-priced} then
gives
\[
\begin{aligned}
  & \bar{J}_{i}( c_{i} ) = \Sigma^{\mu}_{i}( c_{i} ) \\
  &\quad \le
  w_{i} \bigl( J_{i}\bigl( \jet z^{i} \bigr) + L_{J}\, \rho \bigr)
  \\
  &\qquad + \sum_{m}\
  \sum_{\kappa \in \occ^{\rho}_{i}( \jet z^{i}, m )}
  \mu_{i, \kappa, m} - \pi_{i} ,
\end{aligned}
\]
where the equality follows from \eqref{eq:priced-running} and
\eqref{eq:reduced-rent}, as $\Sigma_{i}( c_{i} ) = J_{i}( c_{i} )$.
Finally, $c_{i}$ is feasible for the single-agent problem
\eqref{eq:socp}, as each of its windows satisfies $g_{k} \ge 0$
for every $g \in \preds_{i}$. As the tree retains $c_{i}$,
$\hat{J}_{i} \le \bar{J}_{i}( c_{i} )$. The retained depth-$T$
sequences are likewise feasible for the pricing problem
\eqref{eq:pricing}, so $\bar{J}_{i}^{*} \le \hat{J}_{i}$.

\subsection{Proof of Theorem~\ref{thm:certificate}}
\label{app:certificate}

(i) The proof follows that of \cref{prop:lagrangian}(i), with the
bound of \cref{prop:priced} in place of the reduced cost
\eqref{eq:reduced-curve} of a jet curve. Each $z^{i}(\cdot)$ is a
feasible solution of the single-agent problem \eqref{eq:socp} with
the margin $g( \jet z^{i} ) \ge L_{\preds}\, \rho$ for every
$g \in \preds_{i}$, as \cref{prop:priced} requires. \Cref{prop:priced} thus
gives, for every $i$,
\begin{equation}
\label{eq:cert-agent}
\begin{aligned}
  \hat{J}_{i}
  \;\le\;\; & w_{i} \bigl( J_{i}\bigl( \jet z^{i} \bigr)
  + L_{J}\, \rho \bigr) \\
  & + \sum_{m}\
  \sum_{\kappa \in \occ^{\rho}_{i}( \jet z^{i}, m )}
  \mu_{i, \kappa, m}
  \; - \; \pi_{i} .
\end{aligned}
\end{equation}
We show in two steps below that the right side of
\eqref{eq:cert-agent} summed over $i$ is at most
$\Phi( \jet z^{1}, \dots, \jet z^{N} ) + L_{\Phi}\, \rho
+ b^{\top} \mu - \mathbf{1}^{\top} \pi$.

\noindent \emph{Step one:} We show that no forbidden tuple lies in the
$\rho$-inflated occupancies of the reference,
\begin{equation}
\label{eq:cert-noforbid}
  \sum_{j \in e}
  \mathbf{1}\bigl[ \kappa_{j} \in
  \occ^{\rho}_{j}( \jet z^{j}, m ) \bigr]
  \le | e | - 1
  \qquad \forall\, ( g, m, \kappa ) \in \res .
\end{equation}
In summary, every point of a cell in
$\occ^{\rho}_{j}( \jet z^{j}, m )$ lies within $\rho + S + M h$ of
$\jet z^{j}( t )$ for all $t \in I_{m}$, and, by the joint
margins, $h_{t}$ is nonnegative on tuples of such points. A
forbidden tuple contains points at which $h_{t} < 0$ for some
$t \in I_{m}$, and thus cannot lie in the inflated occupancies.
To formalize, suppose to the contrary that
$\kappa_{j} \in \occ^{\rho}_{j}( \jet z^{j}, m )$ for every
$j \in e$ at some row $r = ( g, m, \kappa ) \in \res$, so that
$d( \xi_{j}, \jet z^{j}( t_{j} ) ) \le \rho$ for some
$\xi_{j} \in \kappa_{j}$ and $t_{j} \in I_{m}$. As
$\kappa \in F_{m}( g )$, also
$h_{t}( ( \xi'_{j} )_{j \in e} ) < 0$ for some $t \in I_{m}$ and
$\xi'_{j} \in \kappa_{j}$. By the triangle inequality and the
$M$-smoothness of $z^{j}$, for each $j \in e$,
\[
\begin{aligned}
  d\bigl( \xi'_{j}, \jet z^{j}( t ) \bigr)
  &\le d( \xi'_{j}, \xi_{j} )
  + d\bigl( \xi_{j}, \jet z^{j}( t_{j} ) \bigr) \\
  &\qquad + d\bigl( \jet z^{j}( t_{j} ), \jet z^{j}( t ) \bigr) \\
  &\le S + \rho + M h ,
\end{aligned}
\]
so that, by the Lipschitz property of $h_{t}$ in the maximum
metric and the joint margins,
\[
  h_{t}\bigl( ( \xi'_{j} )_{j \in e} \bigr)
  \ge h_{t}\bigl( ( \jet z^{j}( t ) )_{j \in e} \bigr)
  - L_{\preds}\, ( \rho + S + M h ) \ge 0 ,
\]
a contradiction. Reordering the summation as in
\eqref{eq:rowsum}, we then have, by \eqref{eq:cert-noforbid},
$\mu \ge 0$, and $b_{r} = | e | - 1$,
\begin{equation}
\label{eq:cert-rent}
\begin{aligned}
  & \sum_{i} \sum_{m}\
  \sum_{\kappa \in \occ^{\rho}_{i}( \jet z^{i}, m )}
  \mu_{i, \kappa, m} \\
  &\quad = \sum_{r = ( g, m, \kappa ) \in \res} \mu_{r}
  \sum_{j \in e}
  \mathbf{1}\bigl[ \kappa_{j} \in
  \occ^{\rho}_{j}( \jet z^{j}, m ) \bigr] \\
  &\quad \le b^{\top} \mu .
\end{aligned}
\end{equation}

\noindent \emph{Step two:} Recall that $w_{i} = \sum_{\ell} \beta_{\ell}\,
\alpha_{\ell i}$, where $\beta \ge 0$ and
$\mathbf{1}^{\top} \beta = 1$. The weighted sums below thus form
convex combinations of the terms maximized in \eqref{eq:aggregator}
and in the definition $L_{\Phi} := \max_{\ell} \sum_{i}
\alpha_{\ell i}\, L_{J}$ of \cref{prop:jcs},
\begin{equation}
\label{eq:cert-cost}
\begin{aligned}
  \sum_{i} w_{i}\, J_{i}\bigl( \jet z^{i} \bigr)
  &= \sum_{\ell} \beta_{\ell} \sum_{i} \alpha_{\ell i}\,
  J_{i}\bigl( \jet z^{i} \bigr) \\
  &\le \max_{\ell} \sum_{i} \alpha_{\ell i}\,
  J_{i}\bigl( \jet z^{i} \bigr) \\
  &= \Phi\bigl( \jet z^{1}, \dots, \jet z^{N} \bigr) , \\
  \sum_{i} w_{i}\, L_{J}
  &= \sum_{\ell} \beta_{\ell} \sum_{i} \alpha_{\ell i}\, L_{J}
  \;\le\; L_{\Phi} .
\end{aligned}
\end{equation}
Summing \eqref{eq:cert-agent} over $i$ and applying
\eqref{eq:cert-rent} and \eqref{eq:cert-cost},
\[
  \sum_{i} \hat{J}_{i}
  \le \Phi\bigl( \jet z^{1}, \dots, \jet z^{N} \bigr)
  + L_{\Phi}\, \rho + b^{\top} \mu - \mathbf{1}^{\top} \pi ,
\]
which rearranges to (i).

(ii) If the dual solution is optimal, then
$\mathbf{1}^{\top} \pi - b^{\top} \mu = \mathrm{DP}( Q )
= \mathrm{LP}( Q )$ by strong duality. The first claim of (ii)
then follows from (i), and the second from $\hat{J}_{i} \ge 0$
for all $i$.

\subsection{Proof of Proposition~\ref{prop:anytime}}
\label{app:anytime}

Both parts follow iteration $n$ of \cref{alg:dcg} line by line.

\smallskip
\noindent (i) \emph{Completeness and consistency.}
\begin{tracelist}
\item[Initialization (line 3):] Every agent initializes its pool with the
  initial agreement, which every agent endorses by (c). The
  initial agreement is thus globally endorsed at time $t_{n}$.
\item[Negotiation (lines 4--12):] An agent broadcasts columns, publishes
  solutions, and endorses solutions only before time
  $t_{n} - \delta_{\mathrm{comm}}$. By (a), every such message reaches every
  agent before time $t_{n}$.
\item[Reception (line 13):] At time $t_{n}$, all agents thus hold the
  same columns and endorsements, and hence the same set of
  globally endorsed solutions. By the initialization, this set contains the
  initial agreement and is nonempty.
\item[Pick (line 14):] All agents thus pick the same best
  globally endorsed solution $\mathbf{c}^{n}$, with ties broken
  by the common rule.
\end{tracelist}
Every globally endorsed solution is either the initial agreement
or was published at line 11 of \cref{alg:dcg} as a
solution of
$\mathrm{JCP}( Q' )$ for the pool $Q'$ of some agent. It thus
satisfies all predicates, and $\mathbf{c}^{n}$ is an agreement.

\smallskip
\noindent (ii) \emph{Certificate at quiescence.}
\begin{tracelist}
\item[Pricing (lines 9--10):] At quiescence, $\hat{J}_{i} \ge 0$
  for all $i$, so that no agent broadcasts a column. The pool $Q$
  is thus the final pool of iteration $n$.
\item[Endorsement (lines 11--12):] Before time $t_{n} - \delta_{\mathrm{comm}}$,
  each agent publishes a solution of $\mathrm{JCP}( Q )$
  at line 11, which every agent endorses at line 12, as its columns
  lie in $Q$. As in (i), it is globally endorsed at
  time $t_{n}$.
\item[Pick (line 14):] Every globally endorsed solution is
  feasible for $\mathrm{JCP}( Q )$, and thus
  costs at least $\mathrm{JCP}( Q )$. 
\end{tracelist}

As $\mathrm{DP}( Q ) = \mathrm{LP}( Q )$ at the optimal dual
solution, subtracting the second bound of \cref{thm:certificate}(ii)
from $\Phi( \mathbf{c}^{n} )$ establishes the claim.

\section*{Acknowledgment}
Claude (Anthropic) was used to write the software for the
experiments and to render the figures and the supplementary video from the
experimental data. The text and proofs were
dictated by the authors, with dense involvement of the same model for
reasoning, editing, and typesetting.

\bibliographystyle{IEEEtran}
\bibliography{refs}

\end{document}